\documentclass{article}

\usepackage{microtype}
\usepackage{graphicx}
\usepackage{subcaption}
\usepackage{booktabs} % for professional tables
\usepackage{amssymb}
\usepackage[toc,page]{appendix}
\graphicspath{{figures/}{images/}}
\newif\ifcomments

\commentsfalse
\ifcomments
\newcommand{\comments}[1]{#1}
\else
\newcommand{\comments}[1]{}
\fi
\usepackage{hyperref}

\usepackage[accepted]{icml2020}

\usepackage{amsmath}
\usepackage{amsthm}
\DeclareMathOperator*{\argmax}{arg\,max}
\newtheorem{theorem}{Theorem}[subsection]

\icmltitlerunning{Estimating Q(s,s') with Deep Deterministic Dynamics Gradients}

\begin{document}

\twocolumn[
\icmltitle{Estimating $Q(s,s')$ with Deep Deterministic Dynamics Gradients}



\icmlsetsymbol{equal}{*}

\begin{icmlauthorlist}
\icmlauthor{Ashley D. Edwards}{uber}
\icmlauthor{Himanshu Sahni}{gt}
\icmlauthor{Rosanne Liu}{uber,dc}
\icmlauthor{Jane Hung}{uber}
\icmlauthor{Ankit Jain}{uber}
\icmlauthor{Rui Wang}{uber}
\icmlauthor{Adrien Ecoffet}{uber}
\icmlauthor{Thomas Miconi}{uber}
\icmlauthor{Charles Isbell}{gt,equal}
\icmlauthor{Jason Yosinski}{uber,dc,equal}
\end{icmlauthorlist}

\icmlaffiliation{dc}{ML Collective}
\icmlaffiliation{uber}{Uber AI Labs}
\icmlaffiliation{gt}{Georgia Institute of Technology, Atlanta, GA, USA}

\icmlcorrespondingauthor{Ashley D. Edwards}{ashedw88@gmail.com}

\icmlkeywords{reinforcement learning, bellman, imitation, transfer, model-based, values, policy gradient}

\vskip 0.3in
]



\printAffiliationsAndNotice{\icmlEqualContribution} % otherwise use the standard text.

\begin{abstract}
  %%%%% Do not edit %%%%%
In this paper, we introduce a novel form of value function, $Q(s, s')$, that expresses the utility of transitioning from a state $s$ to a neighboring state $s'$ and then acting optimally thereafter. In order to derive an optimal policy, we develop a forward dynamics model that learns to make next-state predictions that maximize this value. This formulation decouples actions from values while still learning off-policy. We highlight the benefits of this approach in terms of value function transfer, learning within redundant action spaces, and learning off-policy from state observations generated by sub-optimal or completely random policies. Code and videos are available at \url{http://sites.google.com/view/qss-paper}.
\end{abstract}
  %%%%% Do not edit %%%%% ICML abstract is fixed
%%%%%%%%%%%%%%%%%%%%%%%%%%%%%%%
%     1. Introduction
\section{Introduction}
%%%%%%%%%%%%%%%%%%%%%%%%%%%%%%%
The goal of reinforcement learning is to learn how to act so as to maximize long-term reward. A solution is usually formulated as finding the optimal policy, {\it i.e.}, selecting the optimal action given a state. A popular approach for finding this policy is to learn a function that defines values though actions, $Q(s,a)$, where $\max_a Q(s,a)$ is a state's value and $\argmax_a Q(s,a)$ is the optimal action \cite{sutton1998reinforcement}. We will refer to this approach as QSA. 

Here, we propose an alternative formulation for off-policy reinforcement learning that defines values solely through states, rather than actions. In particular, we introduce $Q(s, s’)$, or simply QSS, which represents the value of transitioning from one state $s$ to a neighboring state $s' \in N(s)$ and then acting optimally thereafter:
\begin{align}
Q(s,s') = r(s, s') + \gamma \max_{s'' \in N(s')}  Q(s',s''). \nonumber
\end{align}

\begin{figure}[t]
    \centering
    \includegraphics[width=.47\linewidth]{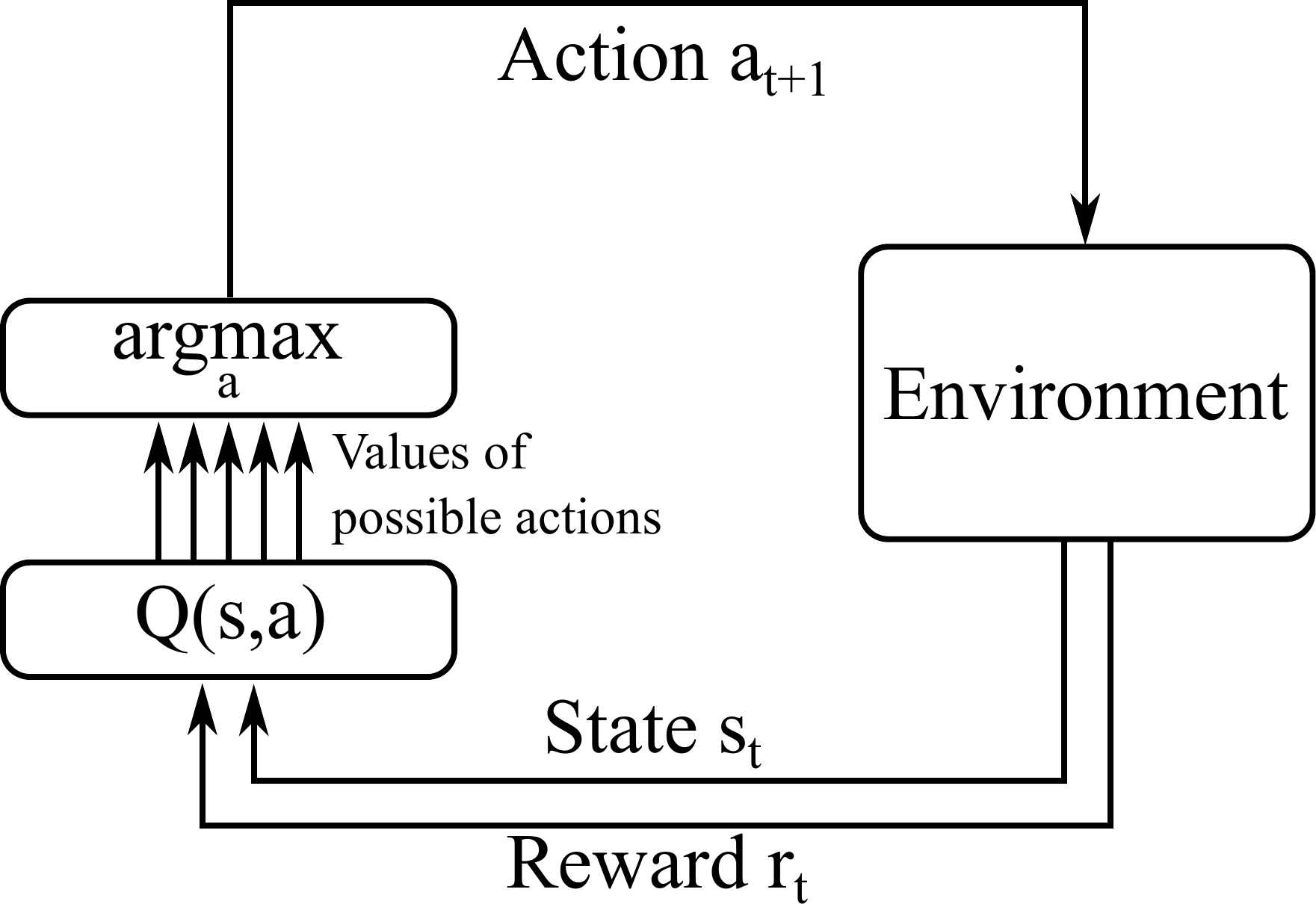}
    \hspace{.04\linewidth}
    \includegraphics[width=.47\linewidth]{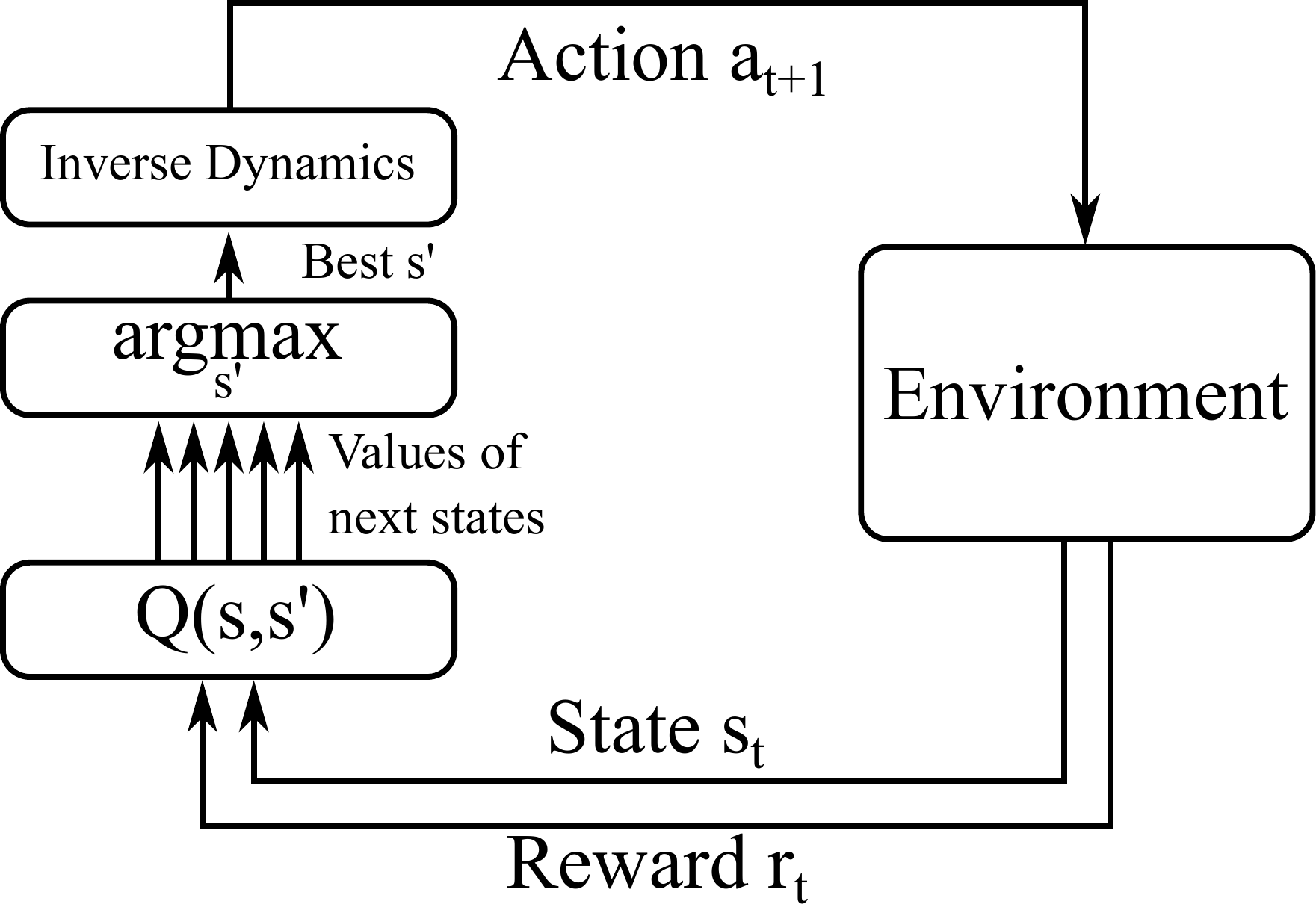} \\
    (a) Q  or QSA-learning \hspace{.15\linewidth} (b) QSS-learning \hspace{.05\linewidth}
    \caption{Formulation for \textbf{(a)} Q-learning, or QSA-learning vs. \textbf{(b)} QSS-learning. Instead of proposing an action, a QSS agent proposes a state, which is then fed into an inverse dynamics model that determines the action given the current state and next state proposal. The environment returns the next observation and reward as usual after following the action.}
    \label{fig:rl} 
\end{figure}
In this formulation, instead of proposing an action, the agent proposes a desired next state, which is fed into an inverse dynamics model that outputs the appropriate action to reach it (see Figure~\ref{fig:rl}).
We demonstrate that this formulation has several advantages. First, redundant actions that lead to the same transition are simply folded into one value estimate. Further, by removing actions, QSS becomes easier to transfer than a traditional Q function in certain scenarios, as it only requires learning an inverse dynamics function upon transfer, rather than a full policy or value function.  
%In this scenario, we could transfer the learned QSS values to the new environment and simply adjust the dynamics function. 
Finally, we show that QSS can learn policies purely from \emph{observations} of (potentially sub-optimal) demonstrations with no access to demonstrator actions. Importantly, unlike other imitation from observation approaches, because it is off-policy, QSS can learn highly efficient policies even from sub-optimal or completely random demonstrations.

In order to realize the benefits of off-policy QSS, we must obtain value maximizing future state proposals without performing explicit maximization. There are two problems one would encounter in doing so. The first is that a set of neighbors of $s$ are not assumed to be known a priori. This is unlike the set of actions in discrete QSA which are assumed to be provided by the MDP. Secondly, for continuous state and action spaces, the set of neighbors may be infinitely many, so maximizing over them explicitly is out of the question. To get around this difficulty, we draw inspiration from Deep Deterministic Policy Gradient (DDPG)~\cite{lillicrap2015continuous}, which learns a policy $\pi(s) \rightarrow a$ over continuous action spaces that maximizes $Q(s,\pi(s))$. We develop the analogous Deep Deterministic Dynamics Gradient (D3G), which trains a forward dynamics model $\tau(s) \rightarrow s'$ to predict next states that maximize $Q(s,\tau(s))$. Notably, this model is not conditioned on actions, and thus allows us to train QSS completely off-policy from observations alone.

We begin the next section by formulating QSS, then describe its properties within tabular settings. We will then outline the case of using QSS in continuous settings, where we will use D3G to train $\tau(s)$. We evaluate in both tabular problems and MuJoCo tasks~\cite{todorov2012mujoco}.

%%%%%%%%%%%%%%%%%%%%%%%%%%%%%%%
%     2. QSS
\section{The QSS formulation for RL}
We are interested in solving problems specified through a Markov Decision Process, which consists of states $s \in S$, actions $a \in A$, rewards $r(s, s') \in R$, and a transition model $T(s, a, s')$ that indicates the probability of transitioning to a specific next state given a current state and action, $P(s' | s, a)$~\cite{sutton1998reinforcement}\footnote{We use $s$ and $s'$ to denote states consecutive in time, which may alternately be denoted $s_t$ and $s_{t+1}$.}. For simplicity, we refer to all rewards $r(s,s')$ as $r$ for the remainder of the paper. Importantly, we assume that the reward function does not depend on actions, which allows us to formulate QSS values without any dependency on actions. 

Reinforcement learning aims to find a policy $\pi(a|s)$ that represents the probability of taking action $a$ in state $s$. We are typically interested in policies that maximize the long-term discounted return $R=\sum_{k=t}^H \gamma^{k-t} r_k$, where $\gamma$ is a discount factor that specifies the importance of long-term rewards and $H$ is the terminal step.

%\subsection{QSA}
Optimal QSA values express the expected return for taking action $a$ in state $s$ and acting optimally thereafter:
\begin{align*}
    Q^*(s,a) = \mathbb{E}[r + \gamma \max_{a'} Q^*(s',a')|s,a].
\end{align*}
These values can be approximated using an approach known as Q-learning~\cite{watkins1992q}:
\begin{align*}
    Q(s,a) \leftarrow Q(s,a) + \alpha [r + \gamma \max_{a'} Q(s',a') - Q(s,a)].
\end{align*}
Finally, QSA learned policies can be formulated as:
\begin{align*}
\pi(s) = \argmax_a Q(s,a).
\end{align*}

%%%%%%%%%%%%%%%%%%%%%%%%%%%%%%%
We propose an alternative paradigm for defining optimal values, $Q^*(s, s')$, or the value of transitioning from state $s$ to state $s'$ and acting optimally thereafter. By analogy with the standard QSA formulation, we express this quantity as:
\begin{equation}
    Q^*(s,s') = r + \gamma \max_{s'' \in N(s')} Q^*(s', s'').
\end{equation}
Although this equation may be applied to any environment, for it to be a useful formulation, the \textit{environment must be deterministic}.
To see why, note that in QSA-learning, the max is over actions, which the agent has perfect control over, and any uncertainty in the environment is integrated out by the expectation.
In QSS-learning the max is over next states, which in stochastic environments are not perfectly predictable. In such environments the above equation does faithfully track a certain value, but it may be considered the ``best possible scenario value'' --- the value of a current and subsequent state assuming that any stochasticity the agent experiences turns out as well as possible for the agent. Concretely, this means we assume that the agent can transition reliably (with probability 1) to any state $s'$ that it is possible (with probability $>$ 0) to reach from state $s$.

Of course, this will not hold for stochastic domains in general, in which case QSS-learning does not track an actionable value. While this limitation may seem severe, we will demonstrate that the QSS formulation affords us a powerful tool for use in deterministic environments, which we develop in the remainder of this article. Henceforth we assume that the transition function is deterministic, and the empirical results that follow show our approach to succeed over a wide range of tasks.

\subsection{Bellman update for QSS}
We first consider the simple setting where we have access to an inverse dynamics model $I(s,s') \rightarrow a$ that returns an action $a$ that takes the agent from state $s$ to $s'$. We also assume access to a function $N(s)$ that outputs the neighbors of $s$. We use this as an illustrative example and will later formulate the problem without these assumptions. 

We define the Bellman update for QSS-learning as:
\begin{align}
    Q(s,s') \leftarrow Q(s,s') + \alpha [r + \gamma \max_{s'' \in N(s)} Q(s',s'') - Q(s,s')].
    \label{eqn:qss_learning}
\end{align}
Note $Q(s,s')$ is undefined when $s$ and $s'$ are not neighbors.
In order to obtain a policy, we define $\tau(s)$ as a function that selects a neighboring state from $s$ that maximizes QSS:
\begin{equation}
    \tau(s) = \argmax_{s' \in N(s)} Q(s,s').
\end{equation}
In words, $\tau(s)$ selects states that have large value, and acts similar to a policy over states.  In order to obtain the policy over actions, we use the inverse dynamics model: 
\begin{equation}
    \pi(s) = I(s, \tau(s)).
\label{eqn:inverse_dynamics}
\end{equation}
This approach first finds the state $s'$ that maximizes $Q(s,s')$, and then uses $I(s,s')$ to determine the action that will take the agent there. We can rewrite Equation~\ref{eqn:qss_learning} as:
\begin{equation}
    Q(s,s') = Q(s,s') + \alpha [r + \gamma Q(s', \tau(s')) - Q(s,s')].
\end{equation}
\begin{figure}[t]
    \centering
      \begin{subfigure}{.32\linewidth}
  	\centering
    \includegraphics[width=\linewidth]{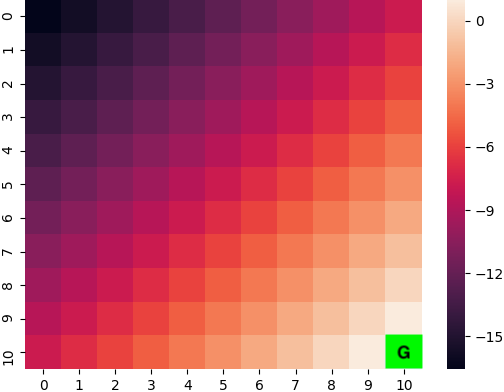}
    \caption{$\max\limits_a Q(s,a)$}
  \end{subfigure}
   \begin{subfigure}{.32\linewidth}
  	\centering
    \includegraphics[width=\linewidth]{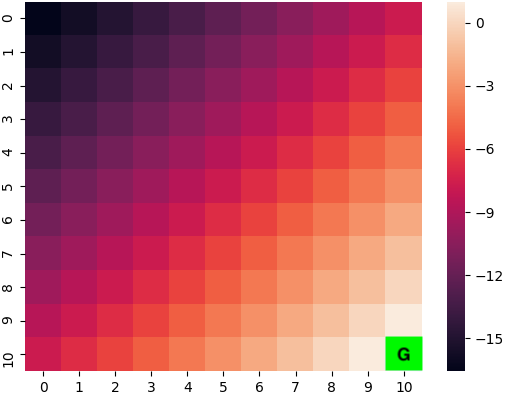}
    \caption{$\max\limits_{s'} Q(s,s')$}
    \label{fig:model_q}
  \end{subfigure}
    \centering
      \begin{subfigure}{.32\linewidth}
  	\centering
    \includegraphics[width=\linewidth]{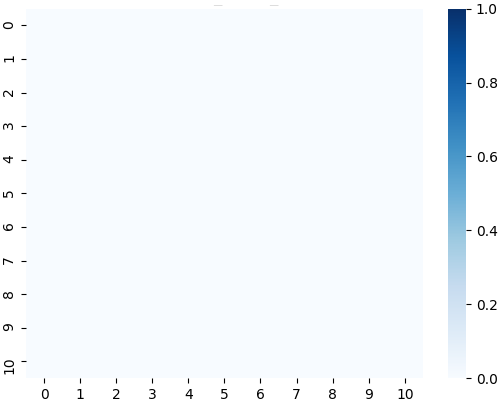}
    \caption{$\frac{QSS - QSA}{|QSS|}$}
  \end{subfigure}
    \caption{Learned values for tabular Q-learning in an 11x11 gridworld. The first two figures show a heatmap of Q-values for QSA and QSS. The final figure represents the fractional difference between the learned values in QSA and QSS.}
    \label{fig:heatmap}
\end{figure}
\subsection{Equivalence of $Q(s,a)$ and $Q(s,s')$}
Let us now investigate the relation between values learned using QSA and QSS. 

\begin{theorem}
QSA and QSS learn equivalent values in the deterministic setting.
\end{theorem}
\begin{proof}
Consider an MDP with a deterministic state transition function and inverse dynamics function $I(s, s')$. QSS can be thought of as equivalent to using QSA to solve the sub-MDP containing only the set of actions returned by $I(s, s')$ for every state $s$:
\begin{equation}
    Q(s, s') = Q(s, I(s, s'))
    \nonumber
\end{equation}
Because the MDP solved by QSS is a sub-MDP of that solved by QSA, there must always be at least one action $a$ for which $Q(s, a) \ge \max_{s'} Q(s, s')$.

The original MDP may contain additional actions not returned by $I(s, s')$, but following our assumptions, their return must be less than or equal to that by the action $I(s, s')$. Since this is also true in every state following $s$, we have:
\begin{equation}
    Q(s, a) \le \max_{s'} Q(s, I(s, s'))\quad\text{for all }a
    \nonumber
\end{equation}
Thus we obtain the following equivalence between QSA and QSS for deterministic environments:
\begin{equation}
    \max_{s'} Q(s, s') = \max_a Q(s, a) 
    \nonumber
\end{equation}
This equivalence will allow us to learn accurate action-values without dependence on the action space.
\end{proof} 
\begin{figure}[t]
    \centering
      \begin{subfigure}{.335\linewidth}
  	\centering
    \includegraphics[width=\linewidth]{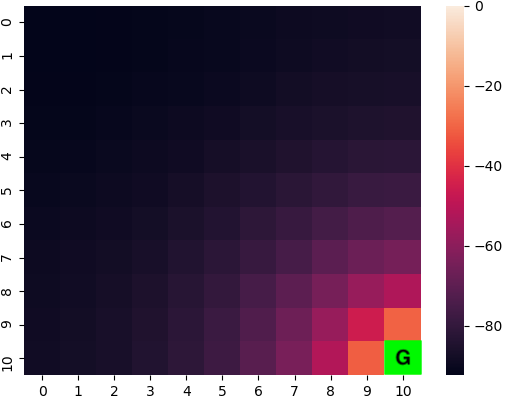}
    \caption{$\max\limits_a Q(s,a)$}
    \label{fig:stochastic_vanilla_q}
  \end{subfigure}
    \begin{subfigure}{.335\linewidth}
  	\centering
    \includegraphics[width=\linewidth]{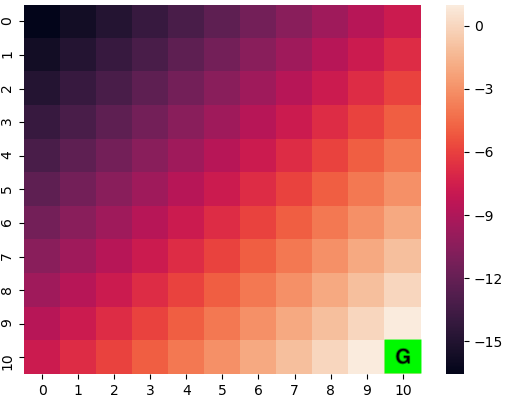}
    \caption{$\max\limits_{s'} Q(s,s')$}
    \label{fig:stochastic_model_q}
  \end{subfigure}
  \centering
      \begin{subfigure}{.3\linewidth}
  	\centering
    \includegraphics[width=\linewidth]{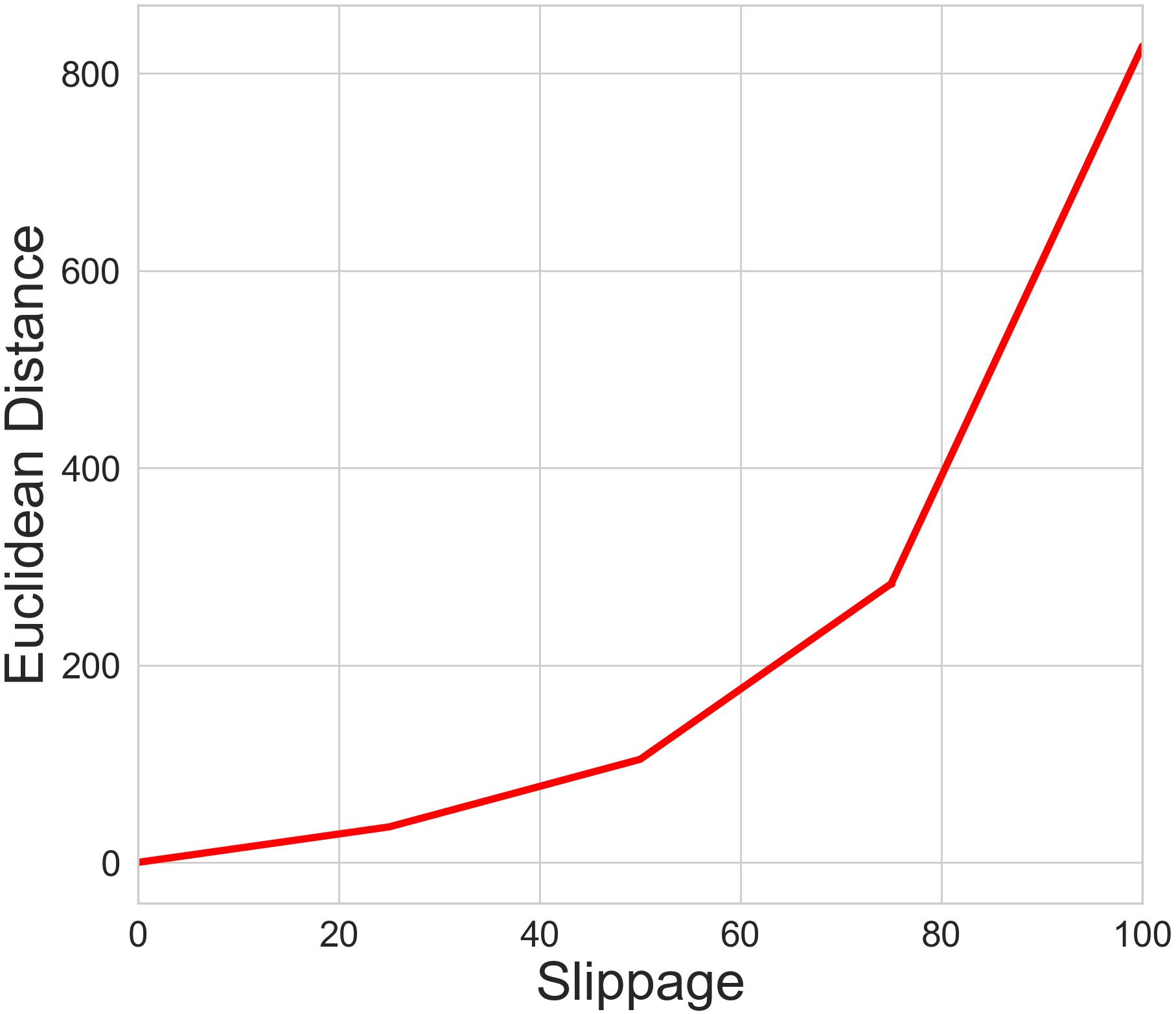}
    \caption{Value distance}
    \label{fig:euclidean_distance}
  \end{subfigure}
    \caption{Learned values for tabular Q-learning in an 11x11 gridworld with stochastic transitions. The first two figures show a heatmap of Q-values for QSA and QSS in a gridworld with 100\% slippage. The final figure represents the euclidean distance between the learned values in QSA and QSS as the transitions become more stochastic (averaged over 10 seeds  with 95\% confidence intervals).}
     \label{fig:stochastic_values}
\end{figure}
%%%%%%%%%%%%%%%%%%%%%%%%%%%%%%%
%     3. QSS: tabular results
\section{QSS in tabular settings}
%%%%%%%%%%%%%%%%%%%%%%%%%%%%%%%
\begin{figure*}[t]
    \centering
      \begin{subfigure}{.24\linewidth}
  	\centering
    \includegraphics[width=\linewidth]{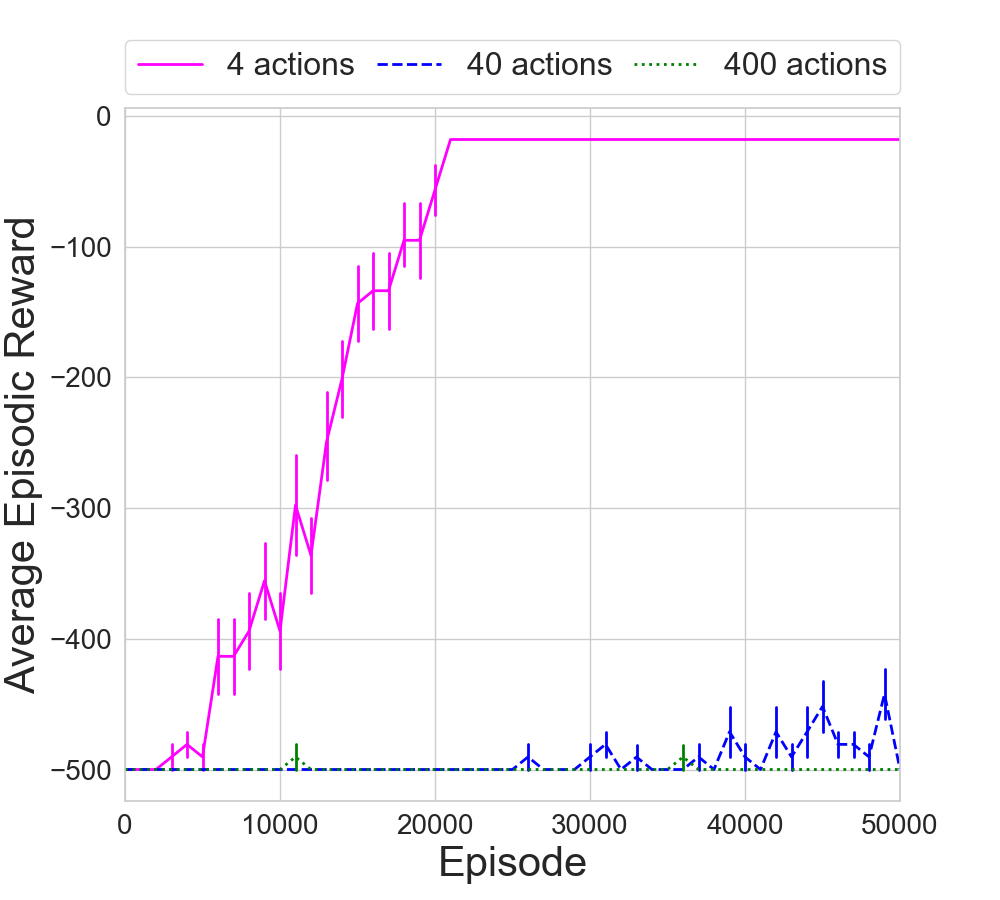}
    \caption{QSA}
    \label{fig:vanilla_actions}
  \end{subfigure}
  \centering
      \begin{subfigure}{.24\linewidth}
  	\centering
    \includegraphics[width=\linewidth]{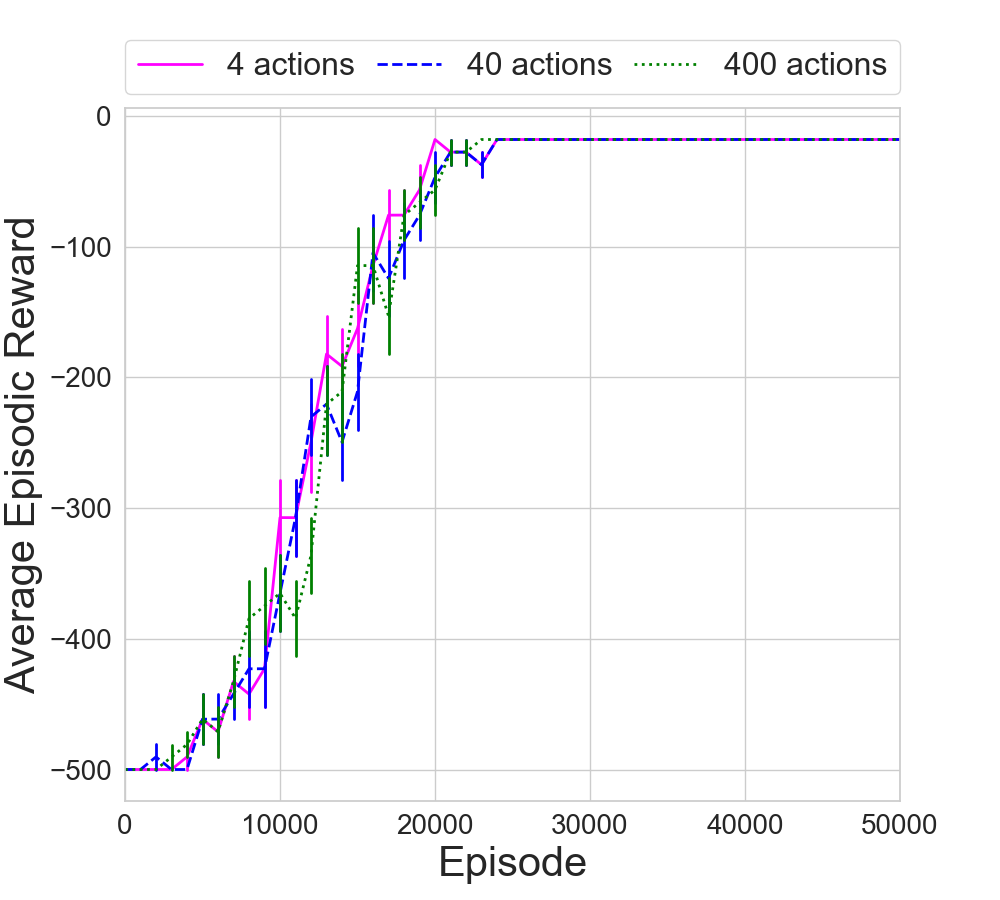}
    \caption{QSS}
    \label{fig:model_actions}
  \end{subfigure}
   \begin{subfigure}{.24\linewidth}
  	\centering
    \includegraphics[width=\linewidth]{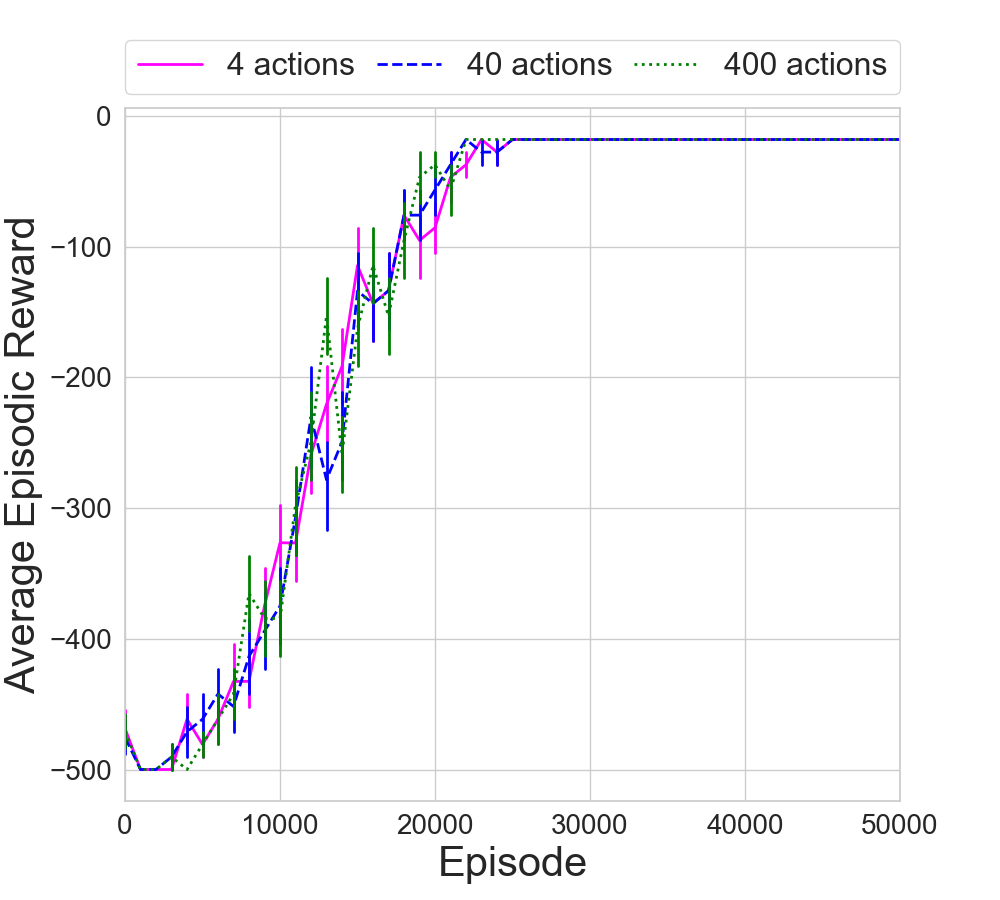}
    \caption{QSS + inverse dynamics}
    \label{fig:id_actions}
  \end{subfigure}
        \begin{subfigure}{.24\linewidth}
  	\centering
    \includegraphics[width=\linewidth]{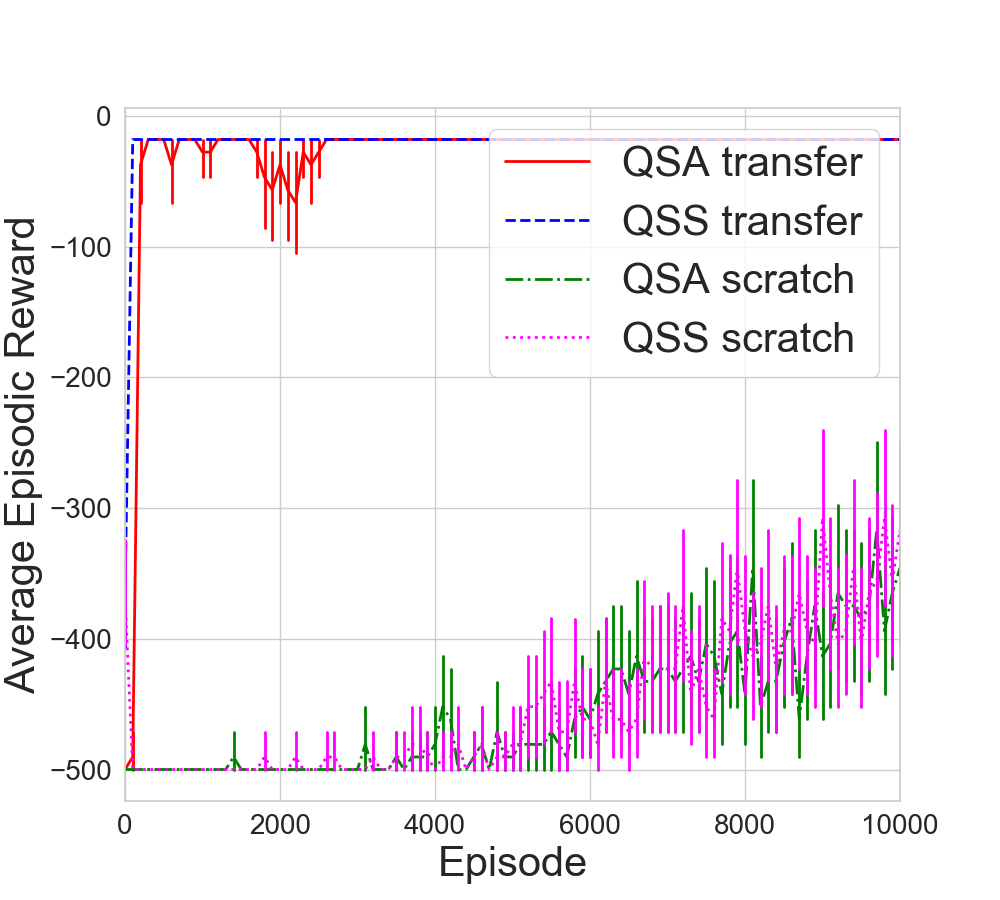}
    \caption{Transfer of permuted actions}
    \label{fig:transfer_actions}
  \end{subfigure}
    \caption{Tabular experiments in an 11x11 gridworld. The first three experiments demonstrate the effect of redundant actions in QSA, QSS, and QSS with learned inverse dynamics. The final experiment represents how well QSS and QSA transfer to a gridworld with permuted actions. All experiments shown were averaged over 50 random seeds with 95\% confidence intervals.}
     \label{fig:redundant_actions}
\end{figure*}
In simple settings where the state space is discrete, $Q(s,s')$ can be represented by a table. We use this setting to highlight some of the properties of QSS. In each experiment, we evaluate within a simple 11x11 gridworld where an agent, initialized at $\langle 0, 0 \rangle$, navigates in each cardinal direction and receives a reward of $-1$ until it reaches the goal. 

\subsection{Example of equivalence of QSA and QSS}
We first examine the values learned by QSS (Figure~\ref{fig:heatmap}). The output of QSS increases as the agent gets closer to the goal, which indicates that QSS learns meaningful values for this task. Additionally, the difference in value between $\max_a Q(s, a)$ and $\max_{s'} Q(s, s')$ approaches zero as the values of QSS and QSA converge. Hence, QSS learns similar values as QSA in this deterministic setting.

\subsection{Example of QSS in a stochastic setting}
The next experiment measures the impact of stochastic transitions on learned QSS values. To investigate this property, we add a probability of slipping to each transition, where the agent takes a random action (i.e. slips into an unintended next state) some percentage of time. First, we notice that the values learned by QSS when transitions have 100\% slippage (completely random actions) are quite different from those learned by QSA (Figure~\ref{fig:stochastic_values}\subref{fig:stochastic_vanilla_q}-\subref{fig:stochastic_model_q}). In fact, the values learned by QSS are similar to the previous experiment when there was no stochasticity in the environment (Figure~\ref{fig:model_q}). As the transitions become more stochastic, the distance between values learned by QSA and QSS vastly increases (Figure~\ref{fig:euclidean_distance}). This provides evidence that the formulation of QSS assumes the best possible transition will occur, thus causing the values to be overestimated in stochastic settings. We include further experiments in the appendix that measure how stochastic transitions affect the average episodic return.

\subsection{QSS handles redundant actions}
One benefit of training QSS is that the transitions from one action can be used to learn values for another action. Consider the setting where two actions in a given state transition to the same next state. QSA would need to make updates for both actions in order to learn their values. But QSS only updates the transitions, thus ignoring any redundancy in the action space. We further investigate this property in a gridworld with redundant actions. Suppose an agent has four underlying actions, up, down, left, and right, but these actions are duplicated a number of times. As the number of redundant actions increases, the performance of QSA deteriorates, whereas QSS remains unaffected (Figure~\ref{fig:redundant_actions}\subref{fig:vanilla_actions}-\subref{fig:model_actions}). 

We also evaluate how QSS is impacted when the inverse dynamics model $I$ is learned rather than given (Figure~\ref{fig:redundant_actions}\subref{fig:id_actions}). We instantiate $I(s,s')$ as a set that is updated when an action $a$ is reached. We sample from this set anytime $I$ is called, and return a random sampling over all redundant actions if $I(s,s')=\emptyset$. Even in this setting, QSS is able to perform well because it only needs to learn about a single action that transitions from $s$ to $s'$.

\subsection{QSS enables value function transfer of permuted actions}
The final experiment in the tabular setting considers the scenario of transferring to an environment where the meaning of actions has changed. We imagine this could be useful in environments where the physics are similar but the actions have been labeled differently. In this case, QSS values should directly transfer, but not the inverse dynamics, which would need to be retrained from scratch. We trained QSA and QSS in an environment where the actions were labeled as 0, 1, 2, and 3, then transferred the learned values to an environment where the labels were shuffled. We found that QSS was able to learn much more quickly in the transferred environment than QSA (Figure~\ref{fig:redundant_actions}\subref{fig:transfer_actions}). Hence, we were able to retrain the inverse dynamics model more quickly than the values for QSA. Interestingly, QSA also learns quickly with the transferred values. This is likely because the Q-table is initialized to values that are closer to the true values than a uniformly initialized value. We include an additional experiment in the appendix where taking the incorrect action has a larger impact on the return. 

%%%%%%%%%%%%%%%%%%%%%%%%%%%%%%%
%     4. D3G: extending to continuous domain
\section{Extending to the continuous domain with D3G}
\label{sec:d3g}
%%%%%%%%%%%%%%%%%%%%%%%%%%%%%%%
%
\begin{algorithm}[t]
\caption{D3G algorithm}
    \begin{algorithmic}[1]
        \STATE \textbf{Inputs:} Demonstrations or replay buffer $D$
        \STATE Randomly initialize $Q_{\theta_1}, Q_{\theta_2},  \tau_\psi,  I_\omega, f_\phi$
        \STATE Initialize target networks $\theta'_1 \leftarrow \theta_1, \theta'_2 \leftarrow \theta_2, \psi' \leftarrow \psi$
        \FOR{$t \in T$}
            \IF{imitation}
                \STATE Sample from demonstration buffer $s, r, s' \sim D$
            \ELSE
                \STATE Take action $a \sim I(s, \tau(s)) + \epsilon$
                \STATE Observe reward and next state
                \STATE Store experience in $D$
                \STATE Sample from replay buffer $s, a, r, s' \sim D$
            \ENDIF
            \STATE
            \STATE Compute $y = r + \gamma \min\limits_{i=1,2}  Q_{\theta'_i} (s', C(s', \tau_{\psi'}(s')))$
            \STATE // Update critic parameters:
            \STATE Minimize $\mathcal{L}_\theta=\sum_i \Vert y - Q_{\theta_i}(s,s') \Vert$
            \STATE
            \IF{$t$ mod $d$}
                \STATE // Update model parameters:
                \STATE Compute $s'_f = C(s, \tau_\psi(s))$
                \STATE Minimize $\mathcal{L}_\psi = -Q_{\theta_1}(s, s'_f) + \beta \Vert \tau_\psi(s) - s'_f)\Vert$
                \STATE // Update target networks:
                \STATE $\theta' \leftarrow \eta \theta + (1-\eta)\theta'$
                \STATE $\psi' \leftarrow \eta \psi + (1-\eta)\psi'$
            \ENDIF
            \STATE
            \IF{imitation}
                \STATE // Update forward dynamics parameters:
                \STATE Minimize $\mathcal{L}_\phi = \Vert f_\phi(s,Q_{\theta'_1}(s,s')) - s' \Vert$
            \ELSE
                \STATE // Update forward dynamics parameters:
                \STATE Minimize $\mathcal{L}_\phi = \Vert f_\phi(s,a) - s' \Vert$
                \STATE // Update inverse dynamics parameters:
                \STATE Minimize $\mathcal{L}_\omega = \Vert I_\omega(s,s') - a\Vert$
            \ENDIF
        \ENDFOR
        
    \end{algorithmic}
    \label{ref:alg_d3g}
\end{algorithm}
\begin{algorithm}[t]
\caption{Cycle}
    \begin{algorithmic}[1]
        \FUNCTION{C($s, s'_\tau$)}
            \IF{imitation}
                \STATE $q = Q_\theta(s, s'_\tau)$ 
                \STATE $s'_f = f_\phi(s, q) $
            \ELSE
                \STATE $a = I_\omega(s, s'_\tau)$ 
                \STATE $s'_f = f_\phi(s, a) $
            \ENDIF
        \ENDFUNCTION
    \end{algorithmic}
    \label{ref:alg_cycle}
\end{algorithm}
In contrast to domains where the state space is discrete and both QSA and QSS can represent relevant functions with a table, in continuous settings or environments with large state spaces we must approximate values with function approximation. One such approach is Deep Q-learning, which uses a deep neural network to approximate QSA~\cite{mnih-2013-arXiv-playing-atari-with,mnih2015human}. The loss is formulated as: $\mathcal{L}_\theta=\Vert y - Q_\theta(s,a) \Vert$, where $y = r + \gamma \max_{a'} Q_{\theta'}(s',a')$. 

Here, $\theta'$ is a target network that stabilizes training. Training is further improved by sampling experience from a replay buffer $s,a,r,s' \sim D$ to decorrelate the sequential data observed in an episode.
  
\subsection{Deep Deterministic Policy Gradients}
Deep Deterministic Policy Gradient (DDPG)~\cite{lillicrap2015continuous} applies Deep Q-learning to problems with continuous actions. Instead of computing a max over actions for the target $y$, it uses the output of a policy that is trained to maximize a critic $Q$: $y = r + \gamma Q_{\theta'}(s, \pi_{\psi'}(s))$. Here, $\pi_\psi(s)$ is known as an actor and trained using the following loss:
\begin{align*}
\mathcal{L}_\psi = -Q_\theta(s, \pi_\psi(s)).
\end{align*}
This approach uses a target network $\theta'$ that is moved slowly towards $\theta$ by updating the parameters as $\theta' \leftarrow \eta \theta + (1-\eta)\theta'$, where $\eta$ determines how smoothly the parameters are updated. A target policy network $\psi'$ is also used when training $Q$, and is updated similarly to $\theta'$. 

\subsection{Twin Delayed DDPG}
Twin Delayed DDPG (TD3) is a more stable variant of DDPG~\cite{fujimoto2018addressing}. One improvement is to delay the updates of the target networks and actor to be slower than the critic updates by a delay parameter $d$. Additionally, TD3 utilizes Double Q-learning~\cite{hasselt2010double} to reduce overestimation bias in the critic updates. Instead of training a single critic, this approach trains two and uses the one that minimizes the output of $y$:
\begin{align*}
    y = r + \gamma \min_{i=1,2} Q_{\theta'_i} (s', \pi_{\psi'}(s')).
\end{align*}
The loss for the critics becomes: 
\begin{align*}
\mathcal{L}_\theta = \sum_i \Vert y - Q_{\theta_i}(s,a) \Vert.
\end{align*}
Finally, Gaussian noise $\epsilon \sim \mathcal{N}(0,0.1)$ is added to the policy when sampling actions. We use each of these techniques in our own approach.
\subsection{Deep Deterministic Dynamics Gradients (D3G)}
A clear difficulty with training QSS in continuous settings is that it is not possible to iterate over an infinite state space to find a maximizing neighboring state. Instead, we propose training a model to directly output the state that maximizes QSS. We introduce an analogous approach to TD3 for training QSS, Deep Deterministic Dynamics Gradients (D3G). Like the deterministic policy gradient formulation $Q(s,\pi_\psi(s))$, D3G learns a model $\tau_\psi(s) \rightarrow s'$ that makes predictions that maximize $Q(s,\tau_\psi(s))$. To train the critic, we specify the loss as:
 \begin{align}
    \mathcal{L}_\theta=\sum_i \Vert y - Q_{\theta_i}(s,s') \Vert.
\end{align}
Here, the target $y$ is specified as:
\begin{align}
y = r + \gamma \min_{i=1,2} Q_{\theta_i'}(s', {\tau}_{\psi'}(s'))].
\end{align}

Similar to TD3, we utilize two critics to stabilize training and a target network for Q.

We train $\tau$ to maximize the expected return, $J$, starting from any state $s$:
\begin{align}
    \nabla_\psi J &= \mathbb{E}[\nabla_\psi Q(s, s')_{s' \sim \tau_\psi(s)}] \\
                  &= \mathbb{E}[\nabla_{s'} Q(s, s') \nabla_\psi \tau_\psi(s)] && \text{[using chain rule]} \nonumber
\end{align}
This can be accomplished by minimizing the following loss:
\begin{align*}
\mathcal{L}_\psi = -Q_\theta(s, \tau_\psi(s)).
\end{align*}
We discuss in the next section how this formulation alone may be problematic. We additionally use a target network for $\tau$, which is updated as $\psi' \leftarrow \eta \psi + (1-\eta)\psi$ for stability. As in the tabular case, $\tau_{\psi}(s)$ acts as a policy over states that aims to maximize $Q$, except now it is being trained to do so using gradient descent. To obtain the necessary action, we apply an inverse dynamics model $I$ as before:
\begin{equation}
    \pi(s) = I_\omega(s,\tau_\psi(s)).
\end{equation}
Now, $I$ is trained using a neural network with data $\langle s,a,s' \rangle \sim D$. The loss is:
\begin{equation}
    \mathcal{L}_\omega = \Vert I_\omega(s,s') - a\Vert.
\end{equation}

\subsubsection{Cycle consistency}
\begin{figure}[t]
    \centering
    \includegraphics[width=\linewidth]{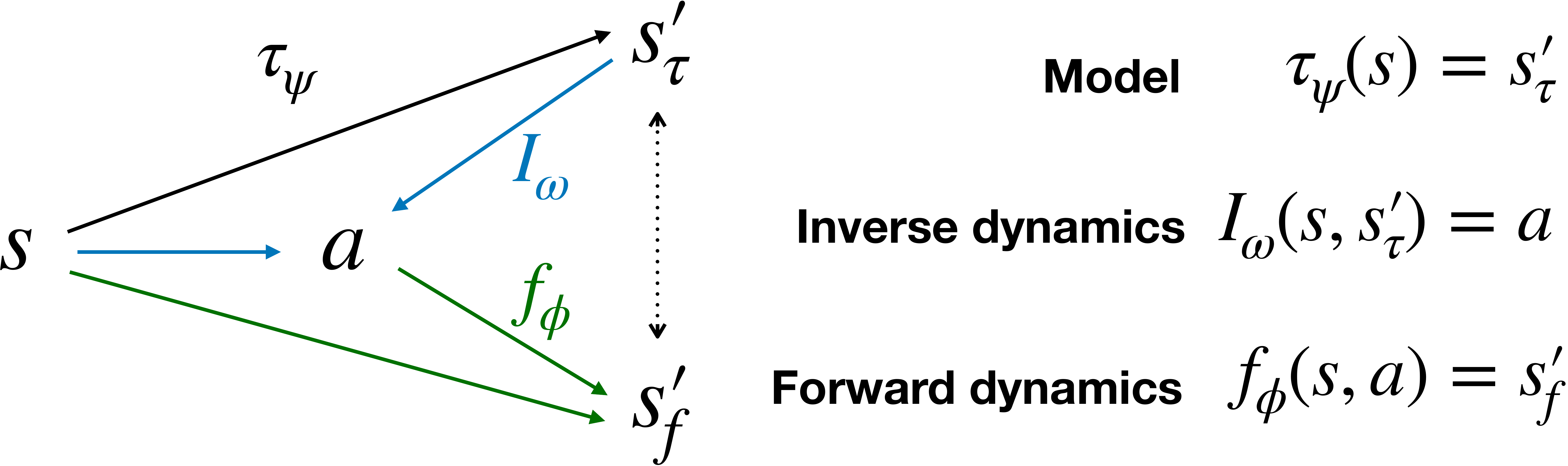}
    \caption{Illustration of the cycle consistency for training D3G. Given a state $s$, $\tau(s)$ predicts the next state $s'_\tau$ (black arrow). The inverse dynamics model $I(s, s'_\tau)$ predicts the action that would yield this transition (blue arrows). Then a forward dynamics model $f_\phi(s,a)$ takes the action and current state to obtain the next state, $s'_f$ (green arrows).
    }
    \label{fig:cycle}
\end{figure}
DDPG has been shown to overestimate the values of the critic, resulting in a policy that exploits this bias~\cite{fujimoto2018addressing}. Similarly, with the current formulation of the D3G loss, $\tau(s)$ can suggest non-neighboring states that the critic has overestimated the value for. To overcome this, we regularize  $\tau$ by ensuring the proposed states are reachable in a single step. In particular, we introduce an additional function for ensuring cycle consistency, $C(s, \tau_\psi(s))$ (see Algorithm~\ref{ref:alg_cycle}). We use this regularizer as a substitute for training interactions with $\tau$. As shown in Figure~\ref{fig:cycle}, given a state $s$, we use $\tau(s)$ to predict the value maximizing next state $s'_\tau$. We use the inverse dynamics model $I(s, s'_\tau)$ to determine the action $a$ that would yield this transition. We then plug that action into a forward dynamics model $f(s,a)$ to obtain the final next state, $s'_f$. In other words, we regularize $\tau$ to make predictions that are consistent with the inverse and forward dynamics models.

To train the forward dynamics model, we compute:
\begin{equation}
    \mathcal{L}_\phi = \Vert f_\phi(s,a) - s' \Vert.
\end{equation}

We can then compute the cycle loss for $\tau_\psi$:
\begin{align}
    \mathcal{L}_\psi = -Q_\theta(s, C(s, \tau_\psi(s)) + \beta \Vert \tau_\psi(s) - C(s, \tau_\psi(s)) \Vert.
\end{align}
The second regularization term further encourages prediction of neighbors. The final target for training Q becomes:
\begin{align}
y = r + \gamma \min_{i=1,2} Q_{\theta_i'}(s', C(s', \tau_{\psi'}(s')))
\end{align}
We train each of these models concurrently. The full training procedure is described in Algorithm~\ref{ref:alg_d3g}.

\subsubsection{A note on training dynamics models}
We found it useful to train the models $\tau_\psi$ and $f_\phi$ to predict the difference between states $\Delta = s' - s$ rather than the next state, as has been done in several other works~\cite{nagabandi2018neural, goyal2018recall,edwards2018forward}. As such, we compute $s'_\tau = s + \tau(s)$ to obtain the next state from $\tau(s)$, and  $s'_f = s + f(s,a)$ to obtain the next state prediction for $f(s,a)$. We describe this implementation detail here for clarity of the paper.
%%%%%%%%%%%%%%%%%%%%%%%%%%%%%%%
%     5. Cool D3G results}
\section{D3G properties and results}
%%%%%%%%%%%%%%%%%%%%%%%%%%%%%%%
%%%%%%%%%%%
\begin{figure}[tb]
    \centering
    \includegraphics[width=.34\linewidth]{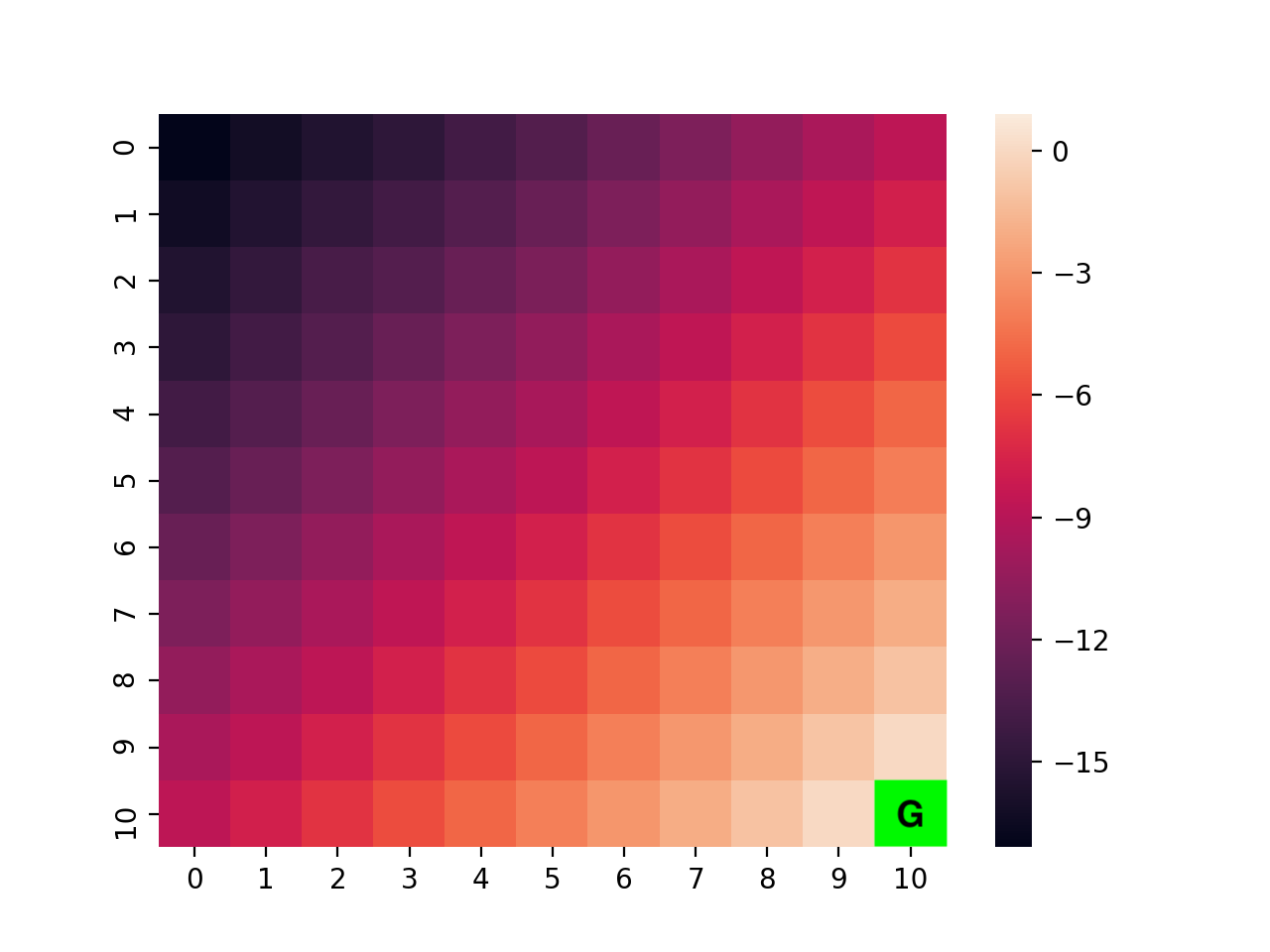}
    \includegraphics[width=.29\linewidth]{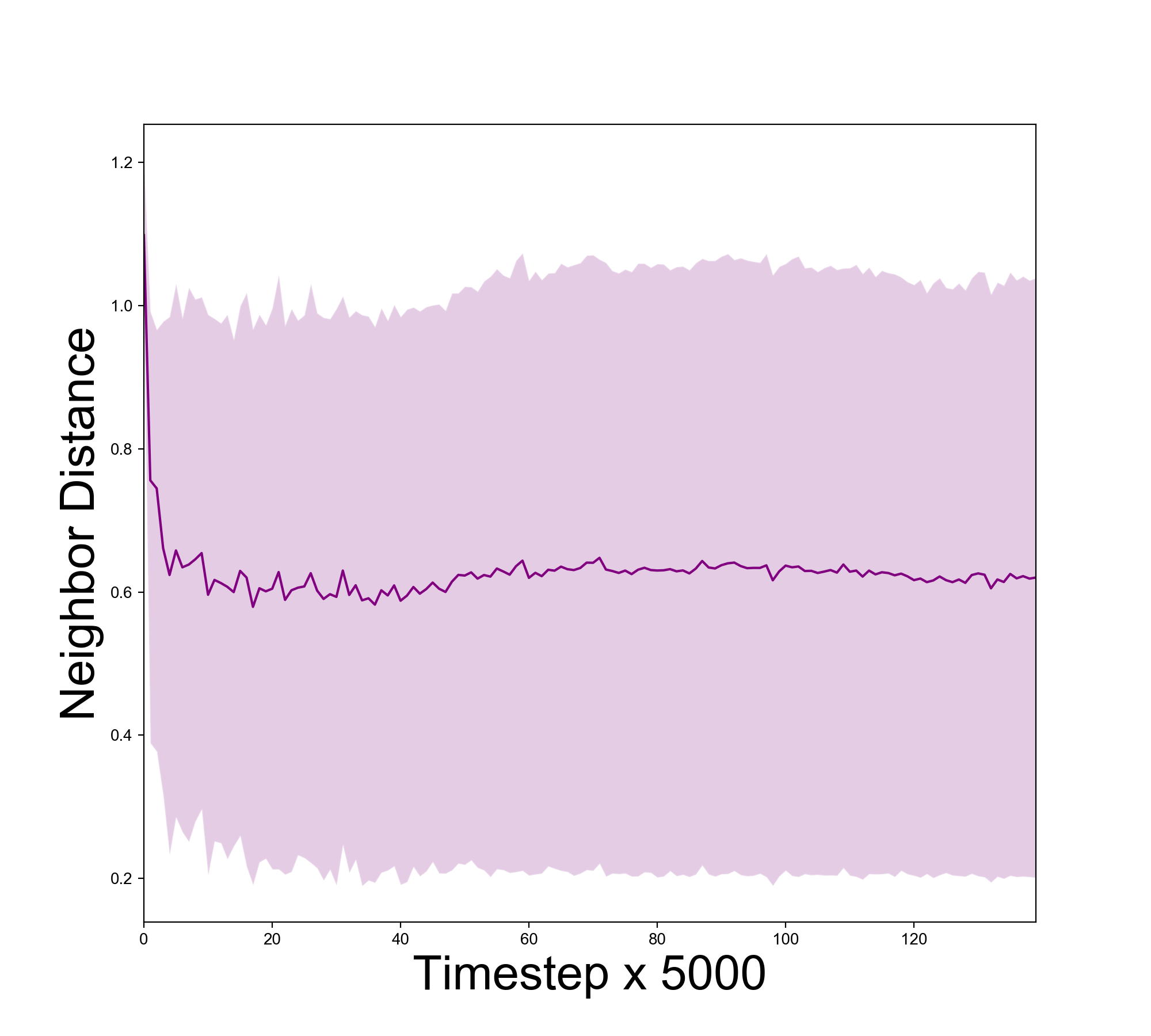}
    \includegraphics[width=.34\linewidth]{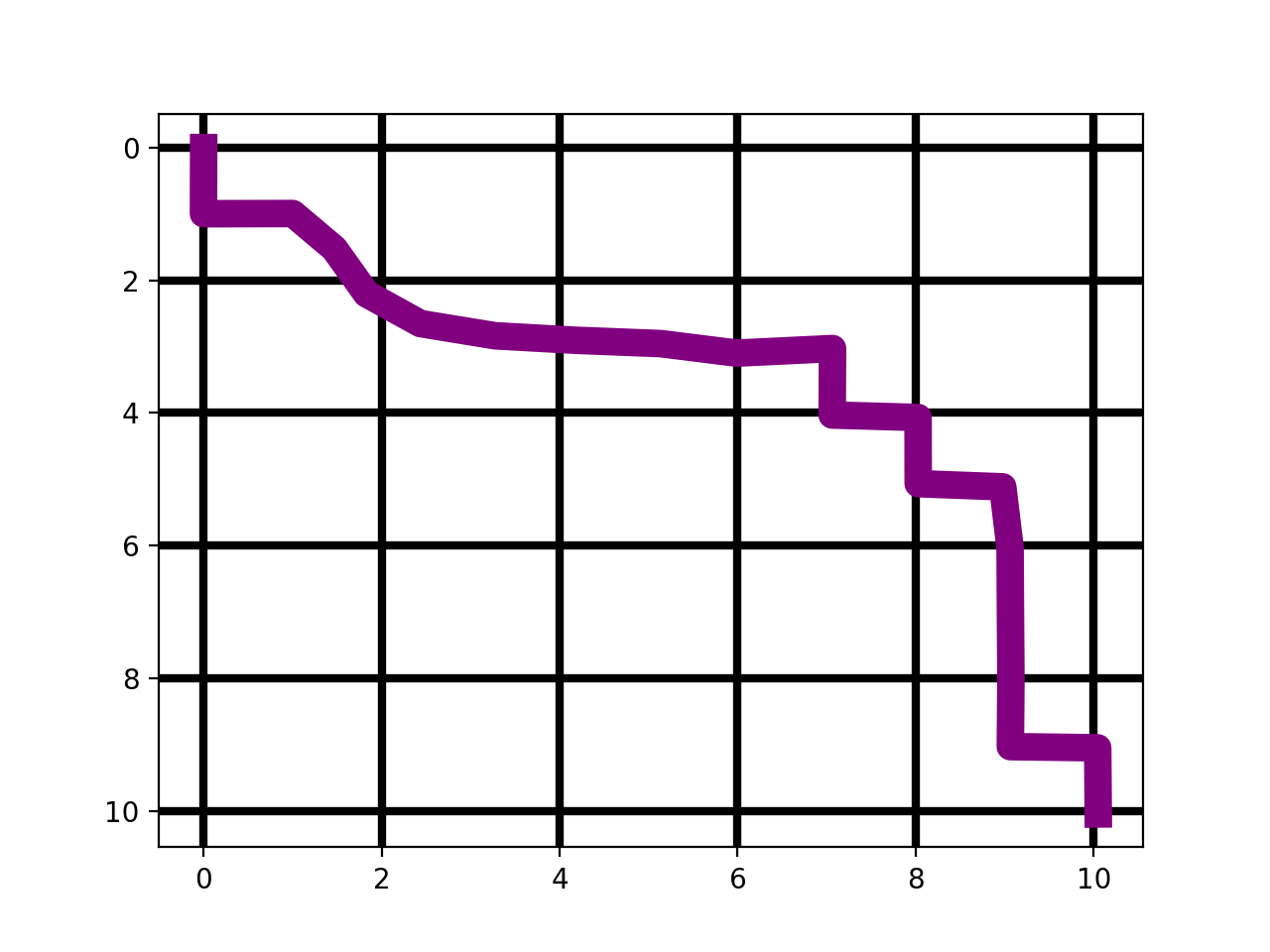}
    
    \includegraphics[width=.34\linewidth]{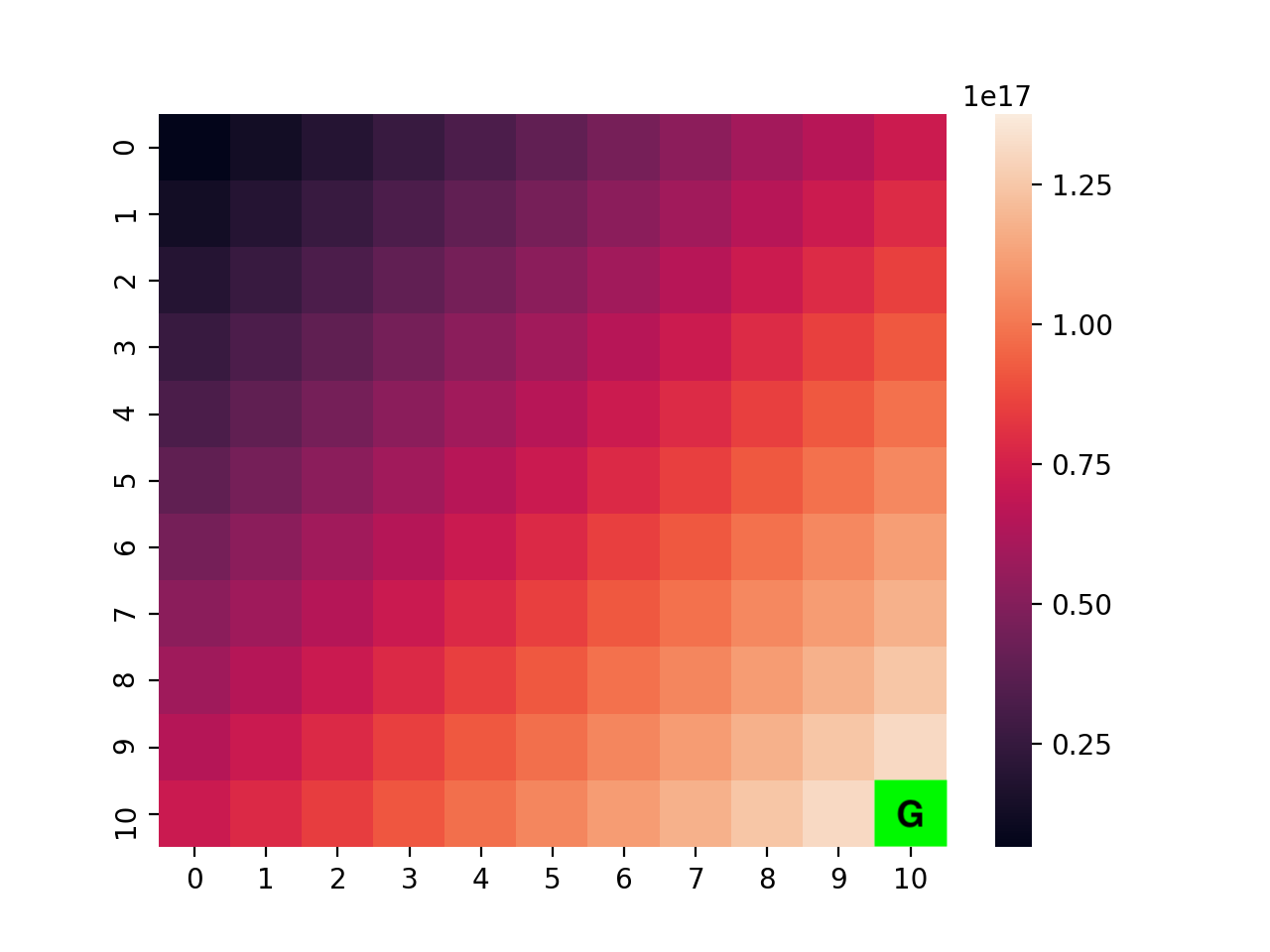}
    \includegraphics[width=.29\linewidth]{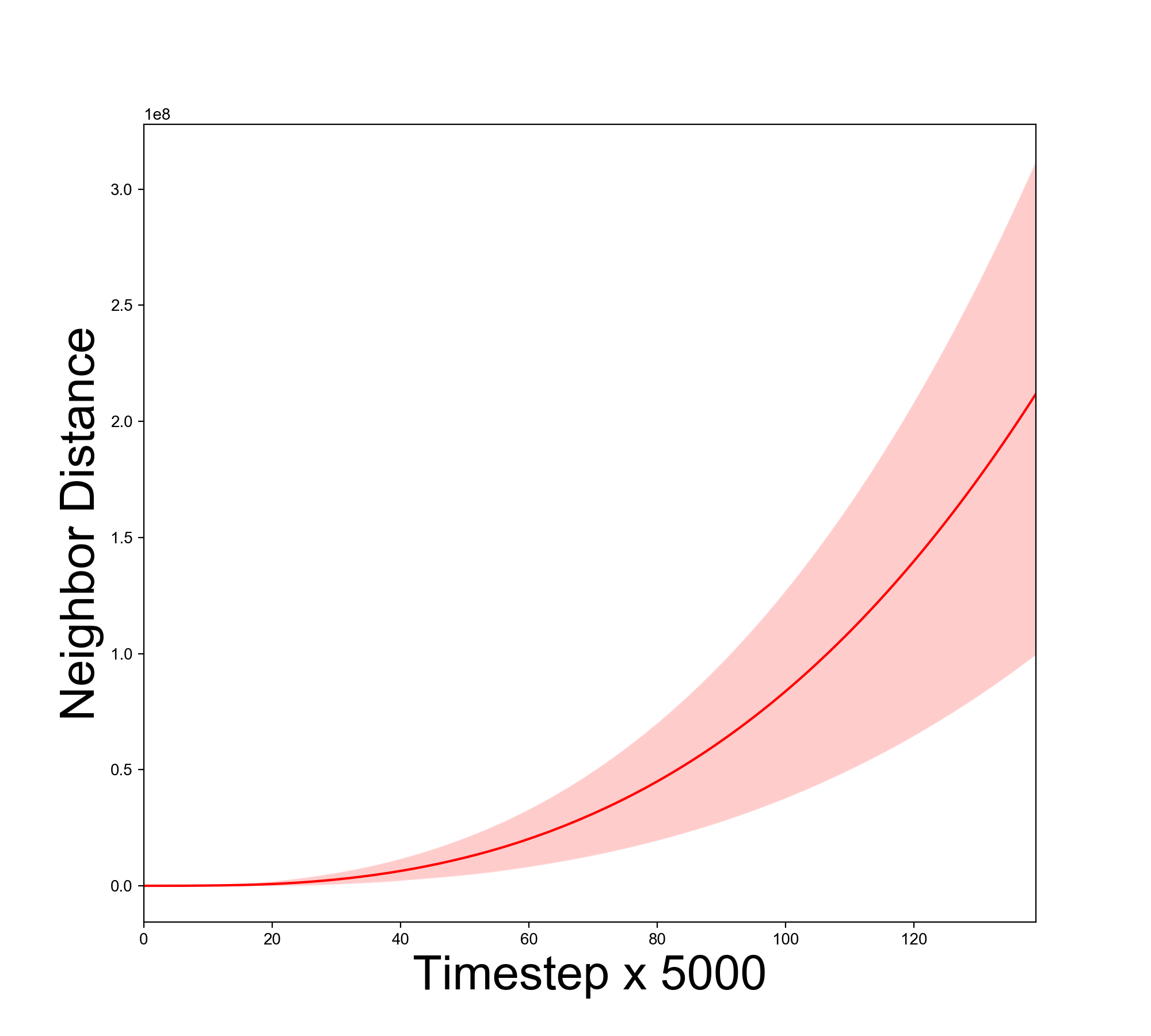}
    \includegraphics[width=.34\linewidth]{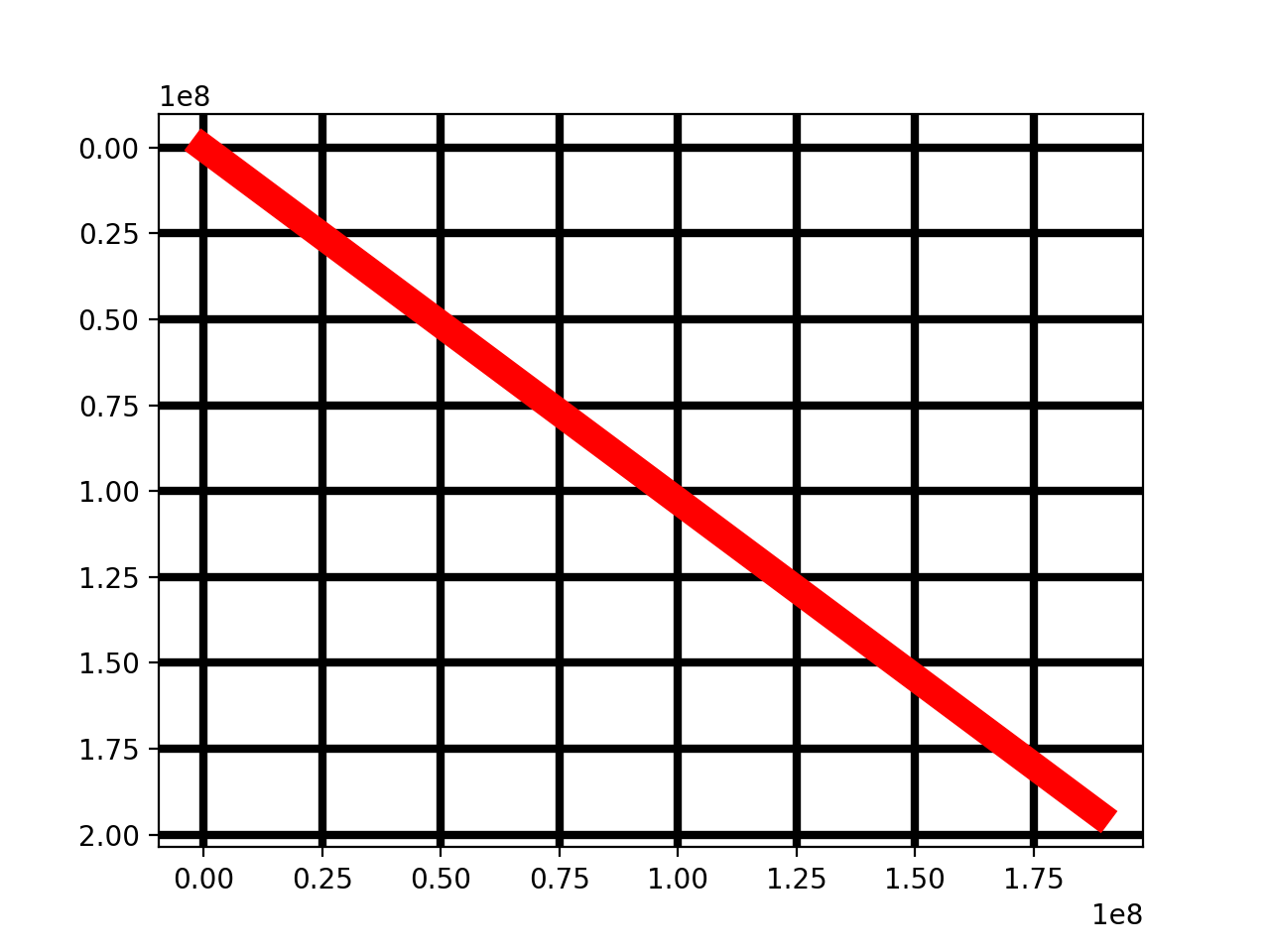}
    \caption{Gridworld experiments for D3G (top) and D3G\textsuperscript{--} (bottom). The left column represents the value function $Q(s,\tau(s))$. The middle column represents the average nearest neighbor predicted by $\tau$ when $s$ was initialized to $\langle 0, 0 \rangle$. These results were averaged over 5 seeds with 95\% confidence intervals. The final column displays the trajectory predicted by $\tau(s)$ when starting from the top left corner of the grid.}
    \label{fig:d3g_grid_experiments}
\end{figure}
 \begin{figure*}[htb]
  \centering
 \includegraphics[width=.245\linewidth]{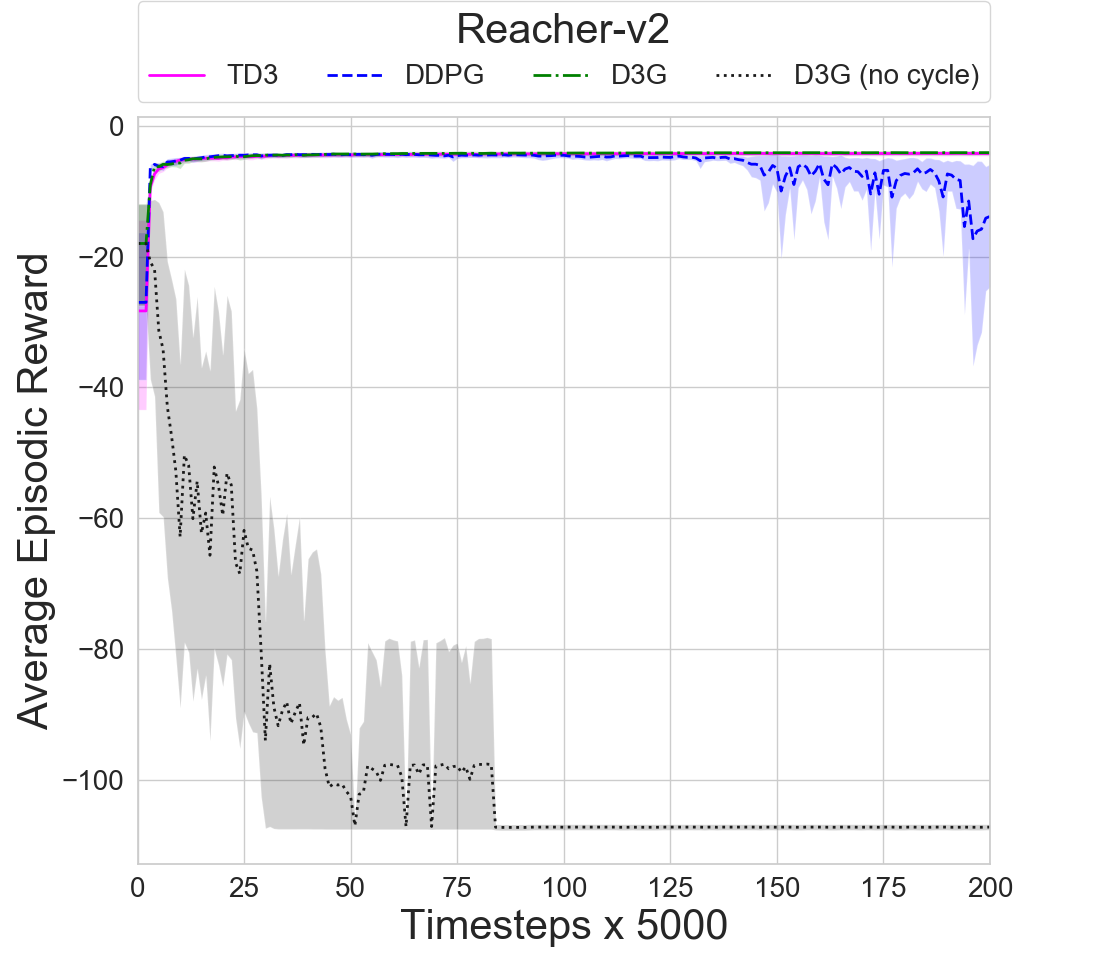}
 %\hspace{.5cm}
 \includegraphics[width=.245\linewidth]{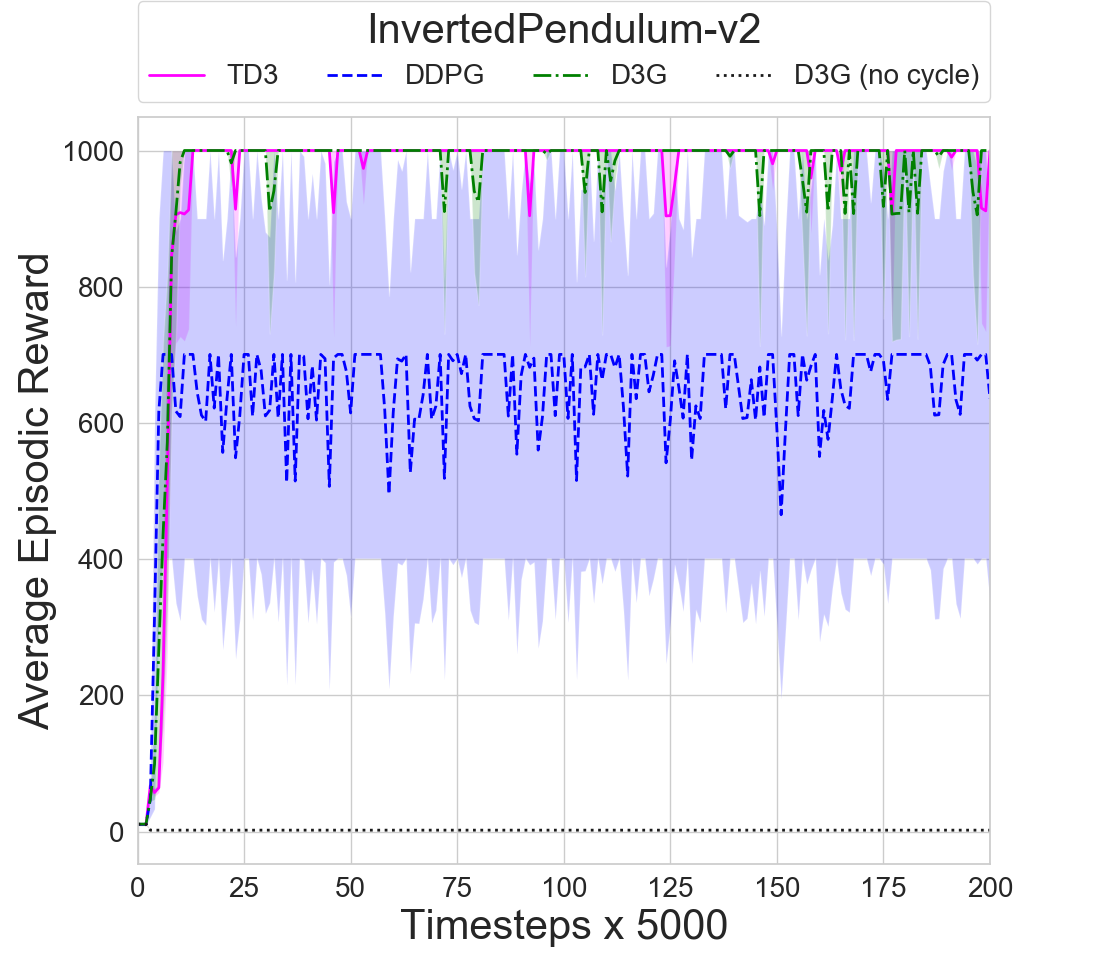}
 %\hspace{.5cm}
 \includegraphics[width=.245\linewidth]{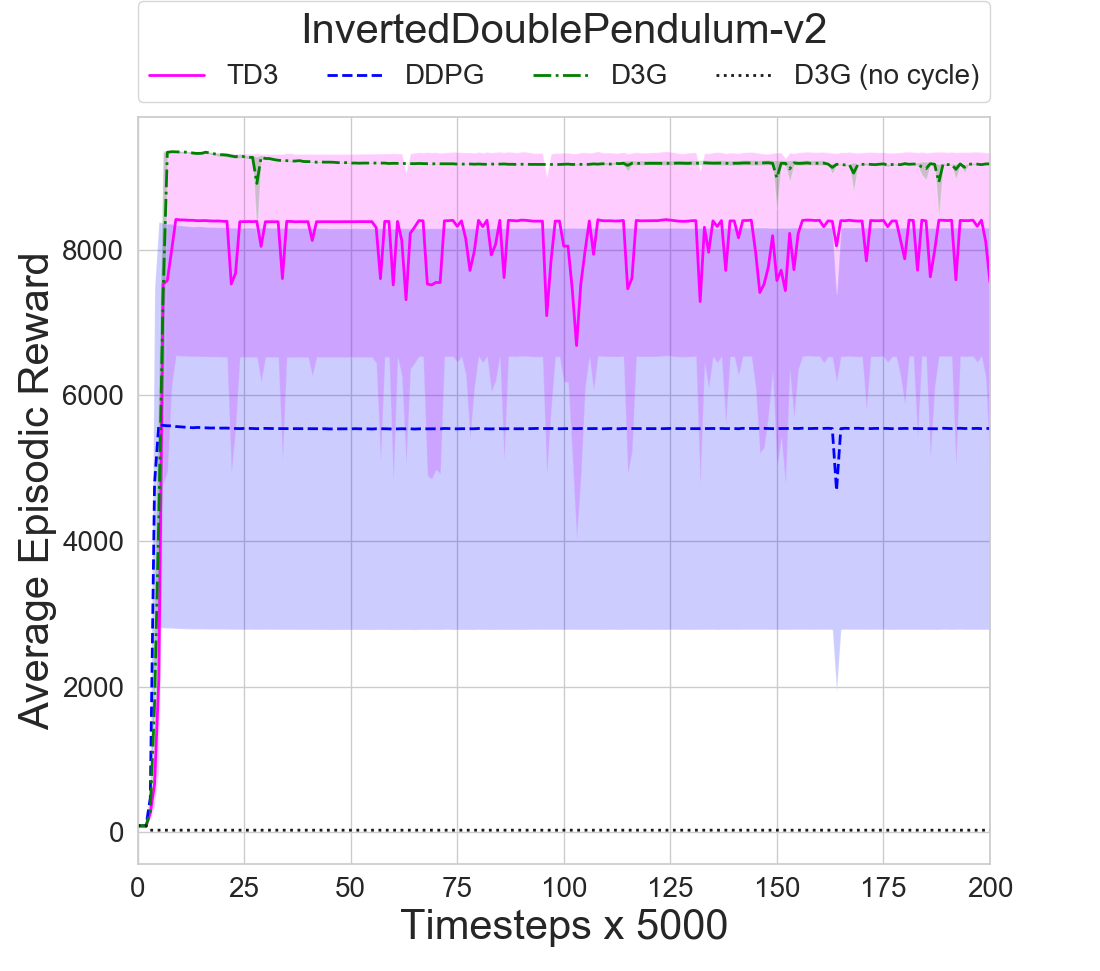}
 %\hspace{.5cm}
 \includegraphics[width=.245\linewidth]{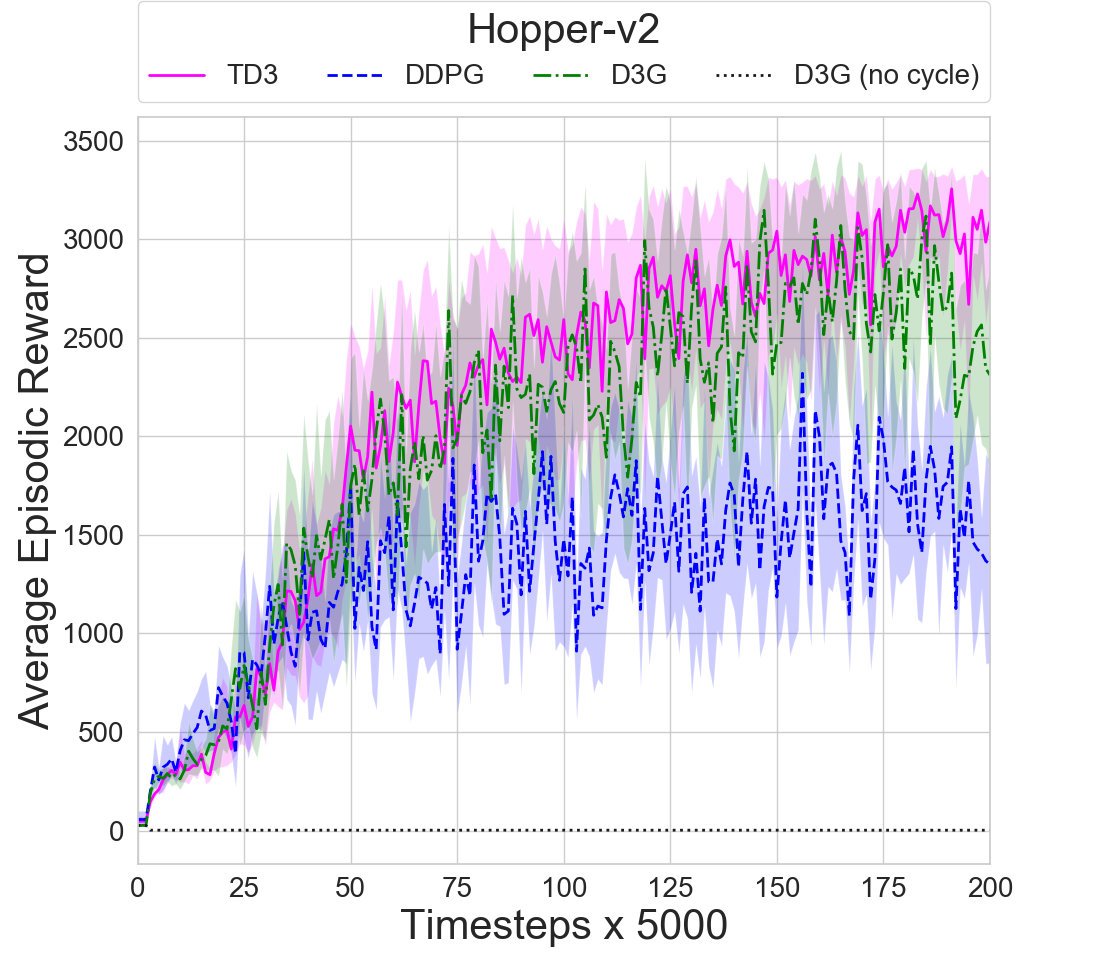}
 
 \includegraphics[width=.245\linewidth]{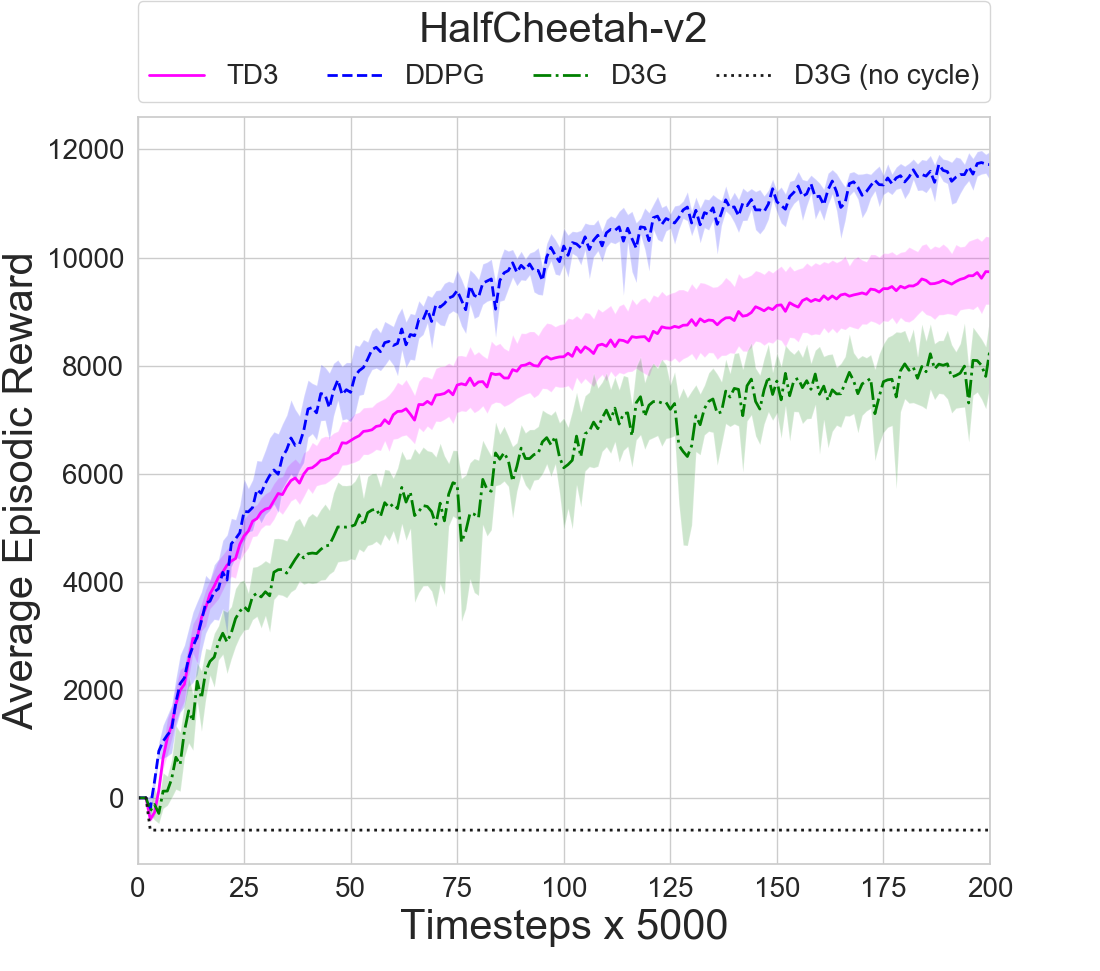}
 %\hspace{.5cm}
 \includegraphics[width=.245\linewidth]{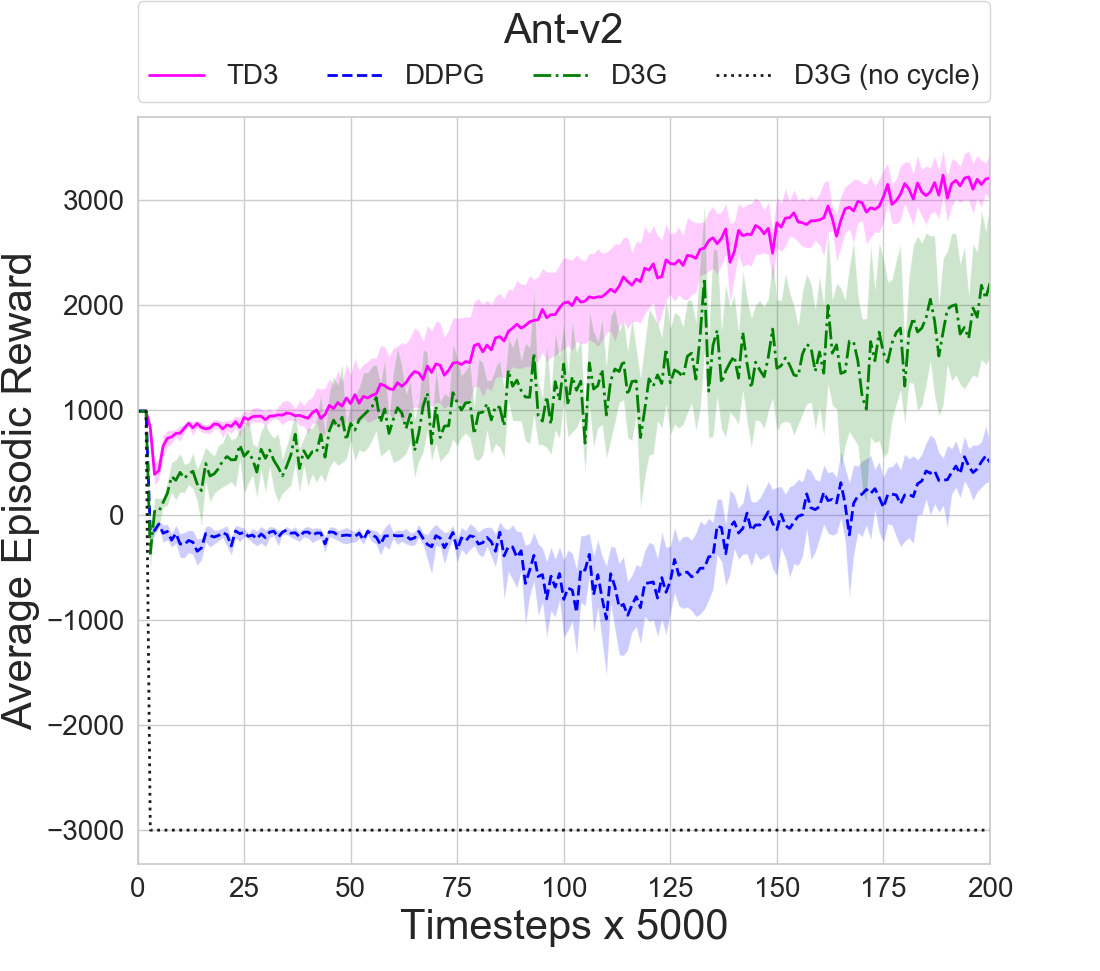}
% \hspace{.5cm}
 \includegraphics[width=.245\linewidth]{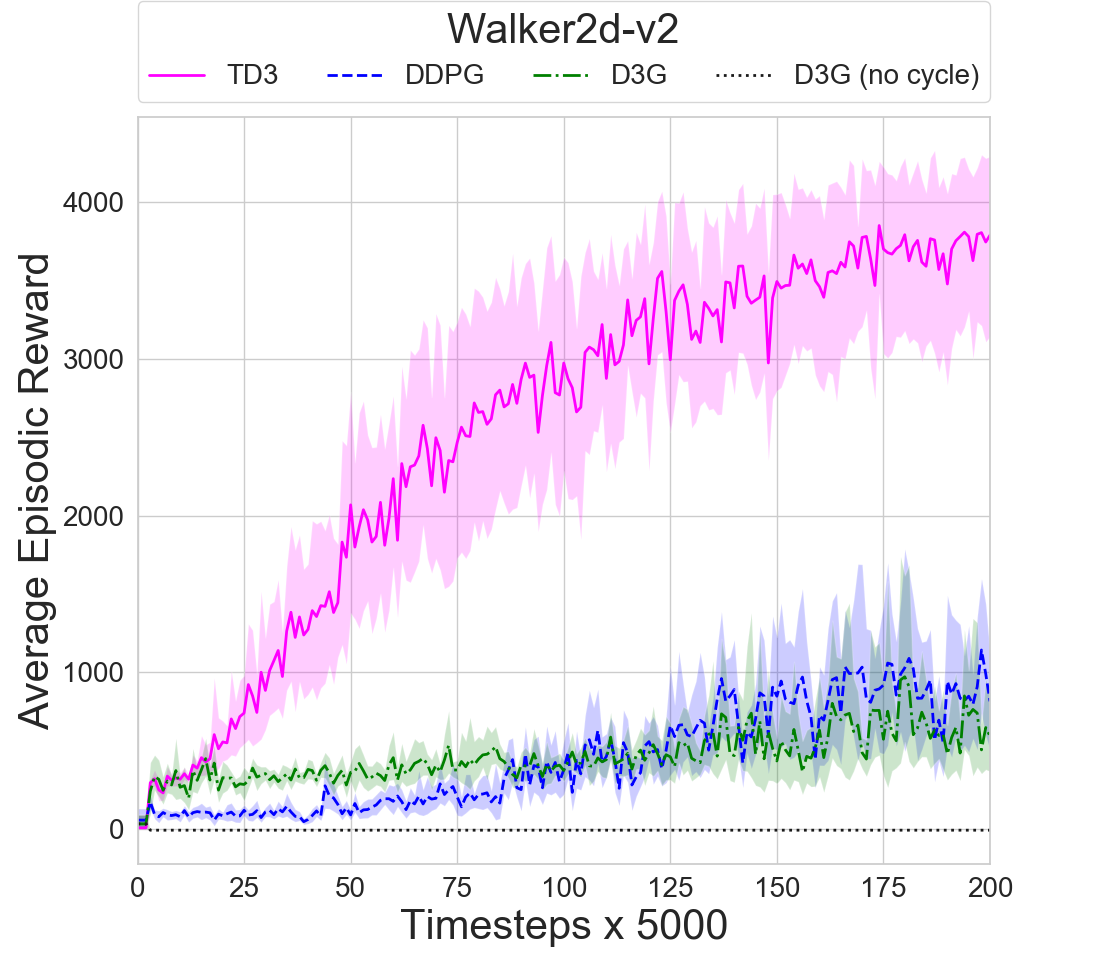}
 %\hspace{.5cm}
 \includegraphics[width=.245\linewidth]{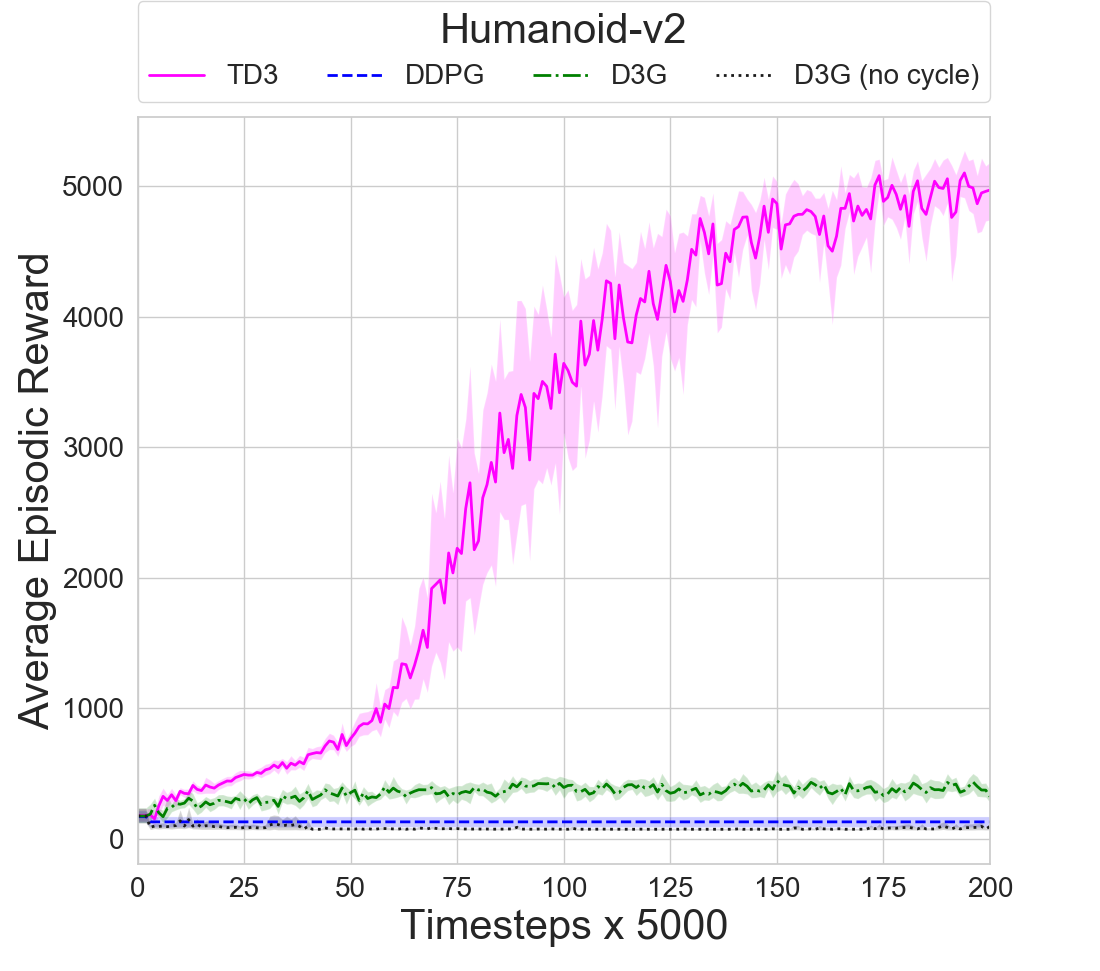}
\caption{Experiments for training TD3, DDPG, D3G\textsuperscript{--}  and D3G in MuJoCo tasks. Every $5000$ timesteps, we evaluated the learned policy and averaged the return over 10 trials. The experiments were averaged over 10 seeds with 95\% confidence intervals.}
\label{fig:mujoco_experiments}
\end{figure*}
\begin{figure*}[t]
    \centering

    \includegraphics[width=.106\linewidth]{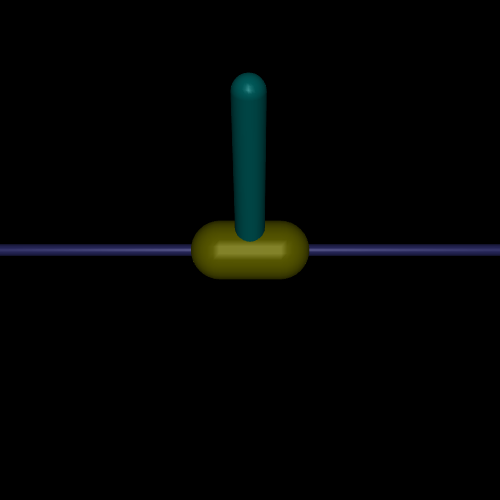}
    \includegraphics[width=.106\linewidth]{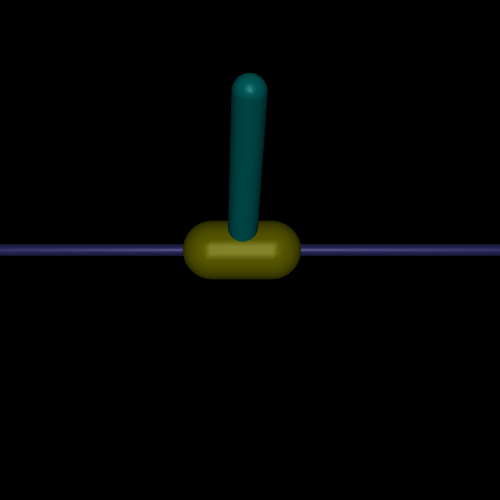}
    \includegraphics[width=.106\linewidth]{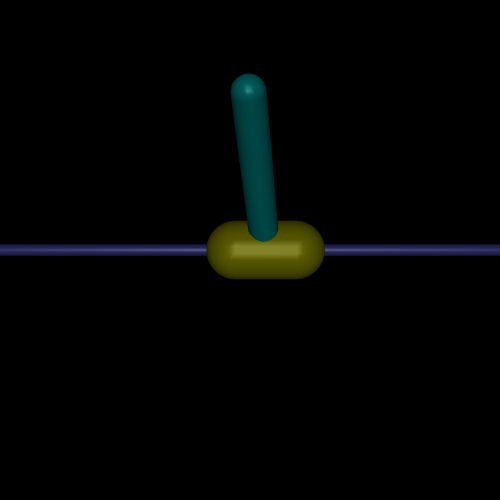}
    \includegraphics[width=.106\linewidth]{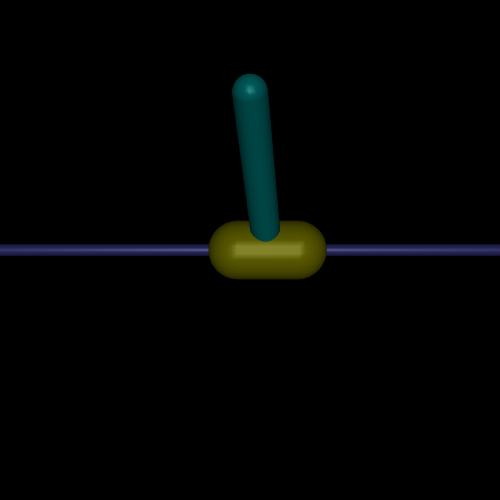}
    \includegraphics[width=.106\linewidth]{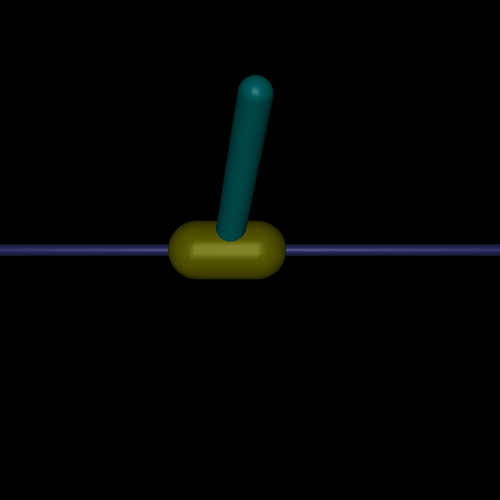}
    \includegraphics[width=.106\linewidth]{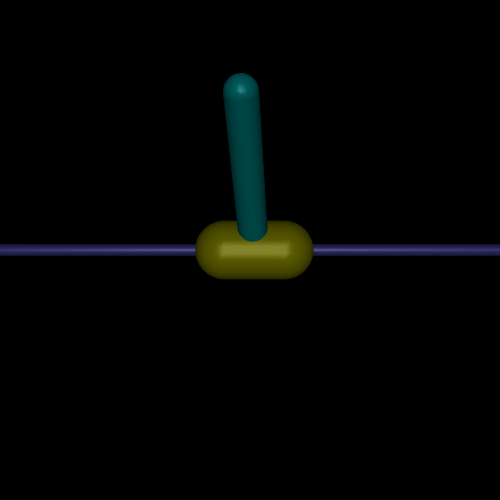}
    \includegraphics[width=.106\linewidth]{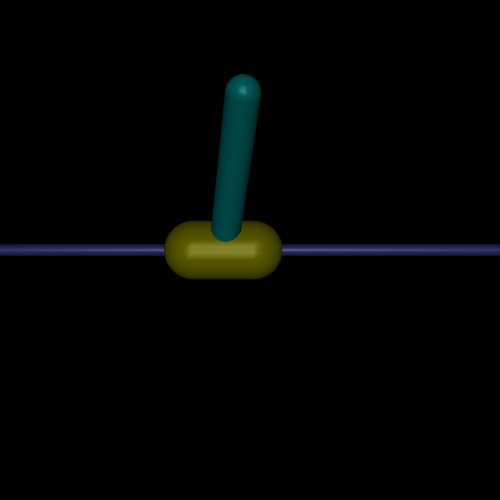}
    \includegraphics[width=.106\linewidth]{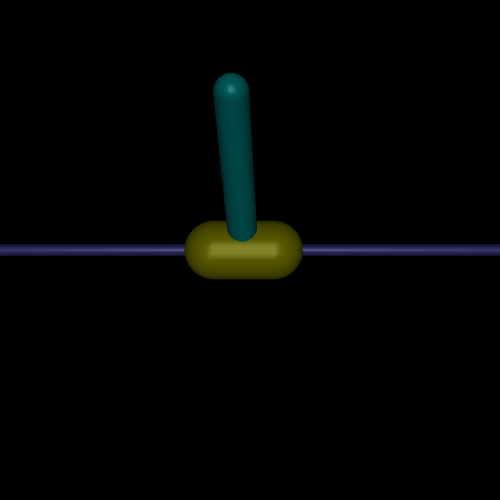}
    \includegraphics[width=.106\linewidth]{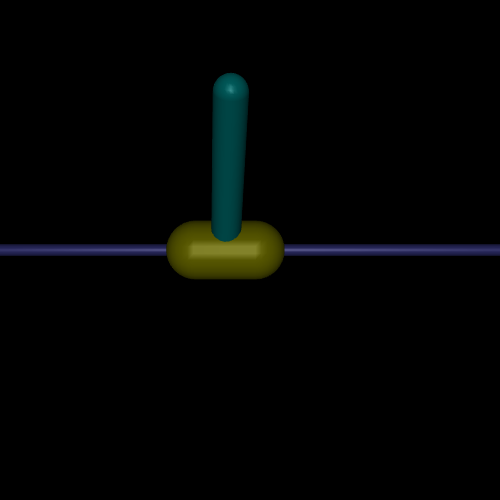}

    \includegraphics[width=.106\linewidth]{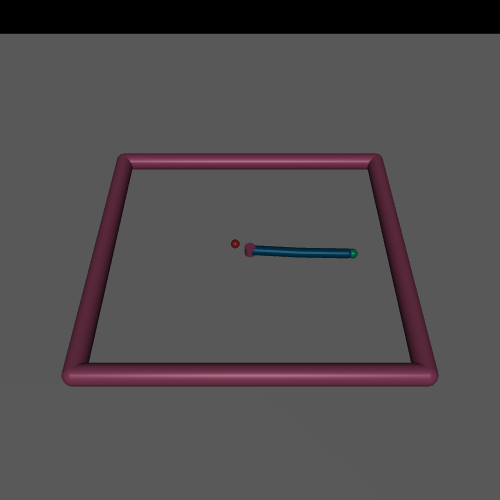}
    \includegraphics[width=.106\linewidth]{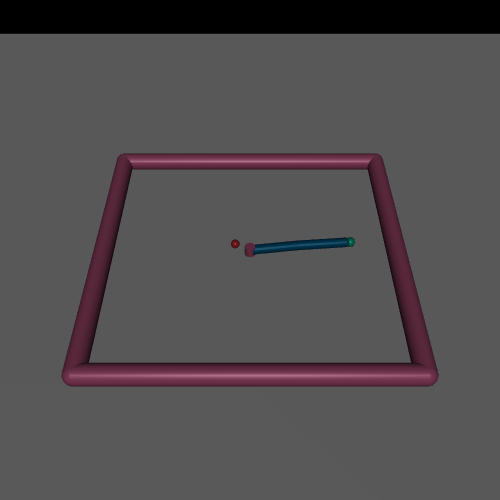}
    \includegraphics[width=.106\linewidth]{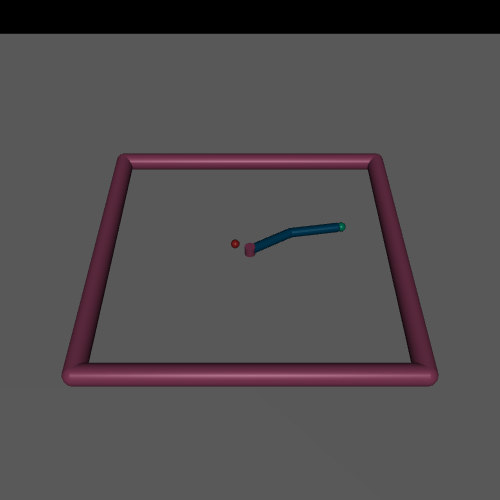}
    \includegraphics[width=.106\linewidth]{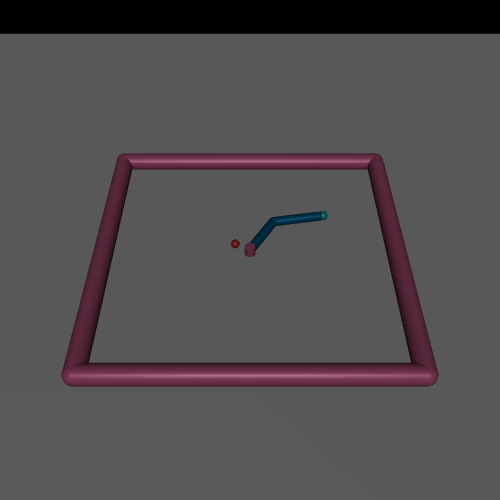}
    \includegraphics[width=.106\linewidth]{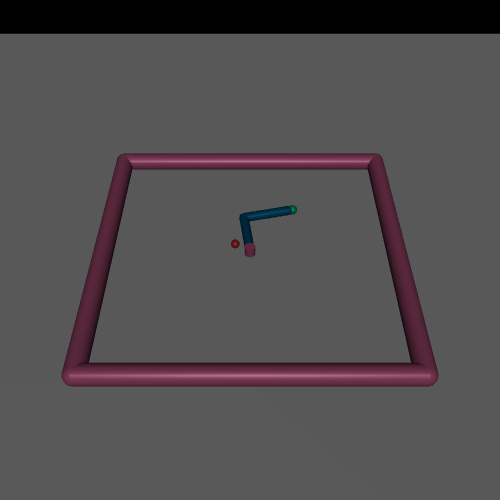}
    \includegraphics[width=.106\linewidth]{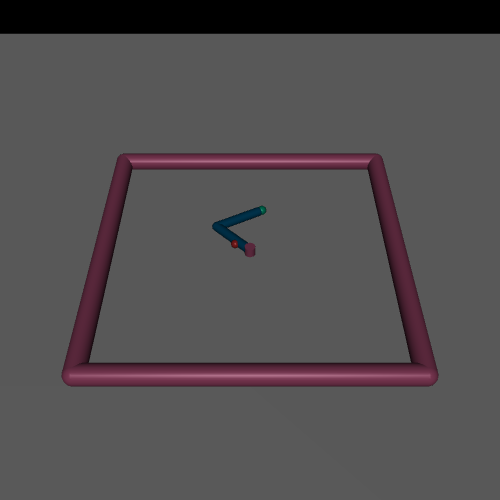}
    \includegraphics[width=.106\linewidth]{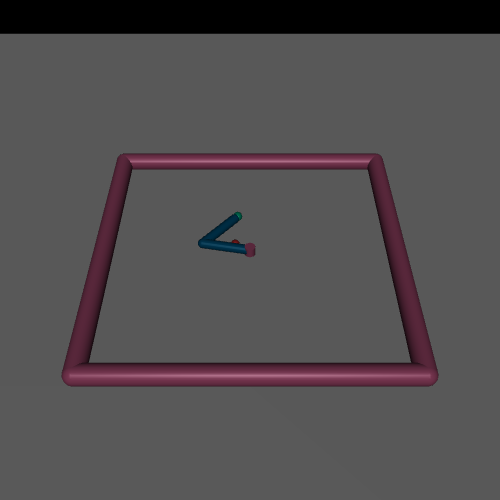}
    \includegraphics[width=.106\linewidth]{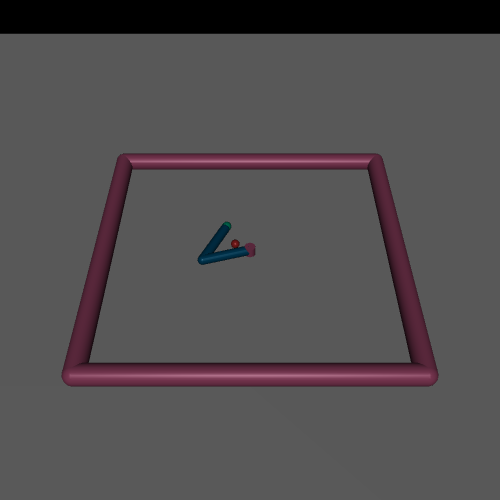}
    \includegraphics[width=.106\linewidth]{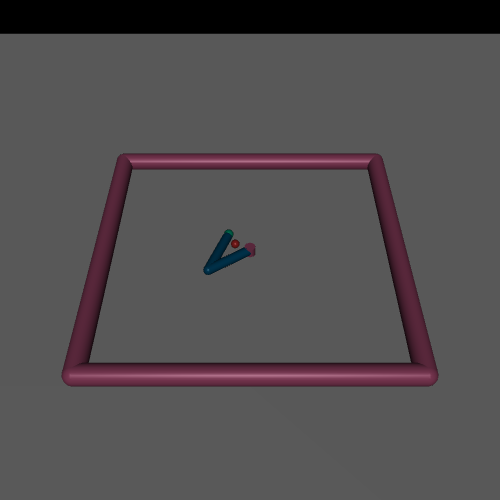}
  
    \caption{D3G generated plans learned from observational data obtained from a completely random policy in InvertedPendulum-v2 (top) and Reacher-v2 (bottom). To generate the plans, we first plugged the initial state from the column in the left into $C(s,\tau(s))$ to predict the next state $s'_f$. We then plugged this state into $C(s'_f, \tau(s'_f))$ to hallucinate the next state. We visualize the model predictions after every $5$ steps. In the Reacher-v2 environment, we set the target (ball position) to be constant and the delta between the fingertip position and target position to be determined by the joint positions (fully described by the first four elements of the state) and the target position. This was only for visualization purposes and was not done during training. Videos are available at \url{http://sites.google.com/view/qss-paper}. }
    \label{fig:mujoco_plans}
\end{figure*}
We now describe several experiments that aimed to measure different properties of D3G. We include full training details of hyperparameters and architectures in the appendix. 

\subsection{Example of D3G in a gridworld}
We first evaluate D3G within a simple 11x11 gridworld with discrete states and actions (Figure~\ref{fig:d3g_grid_experiments}). The agent can move a single step in one of the cardinal directions, and obtains a reward of -1 until it reaches the goal. Because D3G uses an inverse dynamics model to determine actions, it is straightforward to apply it to this discrete setting. 

These experiments examine if D3G learns meaningful values, predicts neighboring states, and makes realistic transitions toward the goal. We additionally investigate the merits of using a cycle loss. 

We first visualize the values learned by D3G and D3G without cycle loss (D3G\textsuperscript{--}). The output of QSS increases for both methods as the agent moves closer to the goal (Figure~\ref{fig:d3g_grid_experiments}). This indicates that D3G can be used to learn meaningful QSS values. However, D3G\textsuperscript{--} vastly overestimates these values\footnote{One seed out of five in the D3G\textsuperscript{--} experiments did yield a good value function, but we did not witness this problem of overestimation in D3G.}. Hence, it is clear that the cycle loss helps to reduce overestimation bias. 

Next, we evaluate if $\tau(s)$ learns to predict neighboring states. First, we set the agent state to $\langle 0, 0 \rangle$. We then compute the minimum Manhattan distance of $\tau(\langle 0, 0 \rangle)$ to the neighbors of $N(\langle 0, 0 \rangle$). This experiment examines how close the predictions made by $\tau(s)$ are to neighboring states. 

In this task, D3G is able to predict states that are no more than one step away from the nearest neighbor on average (Figure~\ref{fig:heatmap}). However, D3G\textsuperscript{--} makes predictions that are significantly outside of the range of the grid. We see this further when
visualizing a trajectory of state predictions made by $\tau$. D3G\textsuperscript{--} simply makes predictions along the diagonal until it extends beyond the grid range. However, QSS learns to predict grid-like steps to the goal, as is required by the environment. This suggests that the cycle loss ensures predictions made by $\tau(s)$ are neighbors of $s$. 

\subsection{D3G can be used to solve control tasks}
\begin{table*}[t]
  \caption{Learning from observation results. We evaluated the learned policies every $1000$ steps for $100000$ total steps. We averaged $10$ trials in each evaluation and computed the maximum average score. We then average of the maximum average scores for $10$ seeds.}
    \label{tab:reacher_bco}
\centering
\begin{tabular}{cc}
    \\
\begin{tabular}{|c|ccc|}
 \cline{2-4}
      \multicolumn{1}{c}{}  & \multicolumn{3}{|c|}{Reacher-v2}\\
\hline
\% Random  & $\pi_o$ & BCO  & D3G  \\  \hline
0  & -$4.1 \pm 0.7$  & \textbf{-}$\boldsymbol{4.2 \pm 0.6}$ & -$14.7 \pm 30.5$        \\ 
25  & -$12.5 \pm 1.0$    & -$4.3 \pm 0.6$ & \textbf{-}$\boldsymbol{4.2 \pm 0.6}$          \\ 
50 & -$22.6 \pm 0.9$    & -$4.9 \pm 0.7$  & \textbf{-}$\boldsymbol{4.2 \pm 0.6}$        \\ 
75  & -$32.6 \pm 0.4$   & -$6.6 \pm 1.3$ & \textbf{-}$\boldsymbol{4.6 \pm 0.6}$          \\ 
100  & -$40.6 \pm 0.5$     & -$9.7 \pm 0.8$   & \textbf{-}$\boldsymbol{6.4 \pm 0.7}$     \\ \hline
\end{tabular}

\begin{tabular}{|ccc|}
  \cline{1-3}
     \multicolumn{3}{|c|}{InvertedPendulum-v2}\\
 \hline
$\pi_o$ & BCO  & D3G  \\ \hline
$1000 \pm 0$   & $\boldsymbol{1000 \pm 0}$ & $3.0 \pm 0.9$         \\ 
$52.3 \pm 3.7$   &  $\boldsymbol{1000 \pm 0}$ & $602.1 \pm 487.4$           \\ 
$18.0 \pm 2.4$    & $12.1 \pm 8.3$  & \textbf{$\boldsymbol{900.2 \pm 299.2}$}           \\ 
$11.4 \pm 1.3$   & $12.1 \pm 8.3$ & \textbf{$\boldsymbol{1000 \pm 0}$}          \\ 
$8.6 \pm 0.3$     & $31.0 \pm 4.7$   & \textbf{$\boldsymbol{1000 \pm 0}$}          \\ \hline
\end{tabular}
\end{tabular}
\end{table*}
We next evaluate D3G in more complicated MuJoCo tasks from OpenAI Gym~\cite{brockman2016openai}. These experiments examine if D3G can be used to learn complex control tasks, and the impact of the cycle loss on training. We compare against TD3 and DDPG.

In several tasks, D3G is able to perform as well as TD3 and significantly outperforms DDPG (Figure~\ref{fig:mujoco_experiments}). Without the cycle loss, D3G\textsuperscript{--} is not able to accomplish any of the tasks. D3G does perform poorly in Humanoid-v2 and Walker2d-v2. Interestingly, DDPG also performs poorly in these tasks. Nevertheless, we have demonstrated that D3G can indeed be used to solve difficult control tasks. This introduces a new research direction for actor-critic, enabling training a dynamics model, rather than policy, whose predictions optimize the return. We demonstrate in the next section that this model is powerful enough to learn from observations obtained from completely random policies. 

\subsection{D3G enables learning from observations obtained from random policies}
Imitation from observation is a technique for training agents to imitate in settings where actions are not available. Traditionally, approaches have assumed that the observational data was obtained from an expert, and train models to match the distribution of the underlying policy~\cite{torabi2018behavioral, edwards2019imitating}. Because $Q(s,s')$ does not include actions, we can use it to~\textit{learn} from observations, rather than imitate, in an off-policy manner. This allows learning from observation data from completely random policies.

To learn from observations, we assume we are given a dataset of state observations, rewards, and termination conditions obtained by some policy $\pi_o$. We train D3G to learn QSS values and a model $\tau(s)$ offline without interacting with the environment. One problem is that we cannot use the cycle loss described in Section~\ref{sec:d3g}, as it relies on knowing the executed actions. Instead, we need another function that allows us to cycle from $\tau(s)$ to a predicted next state. 

To do this, we make a novel observation.~\emph{The forward dynamics model $f$ does not need to take in actions to predict the next state}. It simply needs an input that can be used as a clue for predicting the next state. We propose using $Q(s,s')$ as a replacement for the action. Namely, we now train the forward dynamics model with the following loss:
\begin{equation}
    \mathcal{L}_\phi = \Vert f_\phi(s,Q_{\theta'}(s,s')) - s' \Vert.
\end{equation}
Because Q is changing, we use the target network $Q_{\theta'}$ when learning $f$. We can then use the same losses as before for training QSS and $\tau$, except we utilize the cycle function defined for imitation in Algorithm~\ref{ref:alg_cycle}.

We argue that $Q$ is a good replacement for $a$ because for a given state, different QSS values often indicate different neighboring states. While this may not always be useful (there can of course be multiple optimal states), we found that this worked well in practice.

To evaluate this hypothesis, we trained QSS in InvertedPendulum-v2 and Reacher-v2 with data obtained from expert policies with varying degrees of randomness. We first visualize predictions made by $C(s, \tau(s))$ when trained from a completely random policy (Figure~\ref{fig:mujoco_plans}). Because $\tau(s)$ aims to make predictions that maximize QSS, it is able to hallucinate plans that solve the underlying task. In InvertedPendulum-v2, $\tau(s)$ makes predictions that balance the pole, and in Reacher-v2, the arm moves directly to the goal location. As such, we have demonstrated that $\tau(s)$ can be trained from observations obtained from random policies to produce optimal plans. 

Once we learn this model, we can use it to determine how to act in an environment. To do this, given a state $s$, we use $\tau(s) \rightarrow s'_\tau$ to propose the best next state to reach. In order to determine what action to take, we train an inverse dynamics model $I(s,s'_\tau)$ from a few steps taken in the environment, and use it to predict the action $a$ that the agent should take. We compare this to Behavioral Cloning from Observation (BCO)~\cite{torabi2018behavioral}, which aims to learn policies that mimic the data collected from $\pi_o$. 

As the data collected from $\pi_o$ becomes more random, D3G significantly outperforms BCO, and is able to achieve high reward when the demonstrations were collected from completely random policies (Table~\ref{tab:reacher_bco}). This suggests that D3G is indeed capable of off-policy learning. Interestingly, D3G performs poorly when the data has 0\% randomness. This is likely because off-policy learning requires that every state has some probability of being visited. 

%%%%%%%%%%%%%%%%%%%%%%%%%%%%%%%
%     6. Related Work
\section{Related work}
We now discuss several works related to QSS and D3G.
%%%%%%%%%%%%%%%%%%%%%%%%%%%%%%%

\textbf{Hierarchical reinforcement learning}
The concept of generating states is reminiscent of hierarchical RL~\cite{barto2003recent}, in which the policy is implemented as a hierarchy of sub-policies. In particular, approaches related to feudal RL~\cite{dayan1993feudal} rely on a manager policy providing goals (possibly indirectly, through sub-manager policies) to a worker policy. These goals generally map to actual environment states, either through a learned state representation as in FeUdal Networks~\cite{vezhnevets2017feudal}% Note: The capital U in the middle is not a typo
, an engineered representation as in h-DQN~\cite{kulkarni2016hierarchical}, or simply by using the same format as raw environment states as in HIRO~\cite{nachum2018data}. One could think of the $\tau(s)$ function in QSS as operating like a manager by suggesting a target state, and of the $I(s, s')$ function as operating like a worker by providing an action that reaches that state. Unlike with hierarchical RL, however, both operate at the same time scale. 

\textbf{Goal generation}
This work is also related to goal generation approaches in RL, where a \emph{goal} is a set of desired states, and a policy is learned to act optimally toward reaching the goal. For example, Universal Value Function Approximators~\cite{schaul2015universal} consider the problem of conditioning action-values with goals that, in the simplest formulation, are fixed by the environment. Recent advances in automatic curriculum building for RL reflects the importance of self-generated goals, where the intermediate goals of curricula towards a final objective are automatically generated by approaches such as automatic goal generation \cite{florensa2018automatic}, intrinsically motivated goal exploration processes~\cite{forestier2017intrinsically}, and reverse curriculum generation \cite{pmlr-v78-florensa17a}.

\citet{nair2018visual} employ goal-conditioned value functions along with Variational autoencoders (VAEs) to generate goals for self-supervised practice and for dense reward relabeling in hindsight. Similarly, IRIS \cite{mandlekar2019iris} trains conditional VAEs for goal prediction and action prediction for robot control. \citet{NIPS2019_8818} use a GAN to hallucinate visual goals and combine it with hindsight experience replay \cite{NIPS2017_7090} to increase sample efficiency.
Unlike these approaches, in D3G goals are always a single step away, generated by maximizing the the value of the neighboring state.

\textbf{Learning from observation}
Imitation from Observation (IfO) allows imitation learning without access to actions~\cite{sermanet2017time,liu2017imitation,torabi2018behavioral,edwards2019imitating,torabi2019recent,sun2019provably}. Imitating when the action space differs between the agent and expert is a similar problem, and typically requires learning a correspondence~\cite{kim2019cross,liu2019state}. IfO approaches often aim to match the performance of the expert. D3G aims to~\emph{learn}, rather than imitate. T-REX~\cite{brown2019extrapolating} is a recent IfO approach that can perform better than the demonstrator, but requires a ranking over demonstrations. Finally, like D3G, Deep Q-learning from Demonstrations learns off-policy from demonstration data, but requires demonstrator actions~\cite{hester2018deep}. 

Several works have considered predicting next states from observations, such as videos, which can be useful for planning or video prediction~\cite{finn2017deep,kurutach2018learning, rybkin2018learning,schmeckpeper2019learning}. In our work, the model $\tau$ is trained automatically to make predictions that maximize the return. 

\textbf{Action reduction}
QSS naturally combines actions that have the same effects. Recent works have aimed to express the similarities between actions to learn policies more quickly, especially over large action spaces. For example, one approach is to learn action embeddings, which could then be used to learn a policy~\cite{chandak2019learning, chen2019learning}. Another approach is to directly learn about irrelevant actions and then eliminate them from being selected~\cite{ Zahavy2018LearnWN}. That work is evaluated in the text-based game Zork. Text-based environments would be an interesting direction to explore as several commands may lead to the same next state or have no impact at all. QSS would naturally learn to combine such transitions.

\textbf{Successor Representations}
The successor representation \citep{dayan1993improving} describes a state as the sum of expected occupancy of future states under the current policy. It allows for decoupling of the environment's dynamics from immediate rewards when computing expected returns and can be conveniently learned using TD methods. \citet{barreto2017successor} extend this concept to successor \textit{features}, $\psi^\pi(s,a)$. Successor features are the expected value of the discounted sum of $d$-dimensional \textit{features} of transitions, $\phi(s, a, s')$, under the policy $\pi$. In both cases, the decoupling of successor state occupancy or features from a representation of the reward allows easy transfer across tasks where the dynamics remains the same but the reward function can change. Once successor features are learned, they can be used to quickly learn action values for all such tasks. Similarly, QSS is able to transfer or share values when the underlying dynamics are the same but the action label has changed.

\section{Conclusion}
%%%%%%%%%%%%%%%%%%%%%%%%%%%%%%%
In this paper, we introduced QSS, a novel form of value function that expresses the utility of transitioning to a state and acting optimal thereafter. To train QSS, we developed Deep Deterministic Dynamics Gradients, which we used to train a model to make predictions that maximized QSS. We showed that the formulation of QSS learns similar values as QSA, naturally learns well in environments with redundant actions, and can transfer across shuffled actions. We additionally demonstrated that D3G can be used to learn complicated control tasks, can generate meaningful plans from data obtained from completely random observational data, and can train agents to act from such data.

\section*{Acknowledgements}
The authors thank Michael Littman for comments on related literature and further suggestions for the paper. We would also like to acknowledge Joost Huizinga, Felipe Petroski Such, and other members of Uber AI Labs for meaningful discussions about this work. Finally, we thank the anonymous reviewers for their helpful comments.

\bibliography{icml}
\bibliographystyle{icml2020}

\clearpage
\appendix
\begin{appendices}
\section{QSS Experiments}
We ran all experiments in an 11x11 gridworld. The state was the agent's $\langle x,y \rangle$ location on the grid. The agent was initialized to $\langle 0, 0 \rangle$ and received a reward of $-1$ until it reached the goal at $\langle 10, 10 \rangle$ and obtained a reward of $1$ and was reset to the initial position. The episode automatically reset after $500$ steps. 

We used the same hyperparameters for QSA and QSS. We initialized the Q-values to $.001$. The learning rate $\alpha$ was set to $.01$ and the discount factor was set to $.99$. The agent followed an $\epsilon$-greedy policy. Epsilon was set to $1$ and decayed to $.1$ by subtracting 9e-6 every time step. 

\subsection{Additional stochastic experiments}
\begin{figure}[!htb]
    \centering
      \begin{subfigure}{.325\linewidth}
  	\centering
    \includegraphics[width=\linewidth]{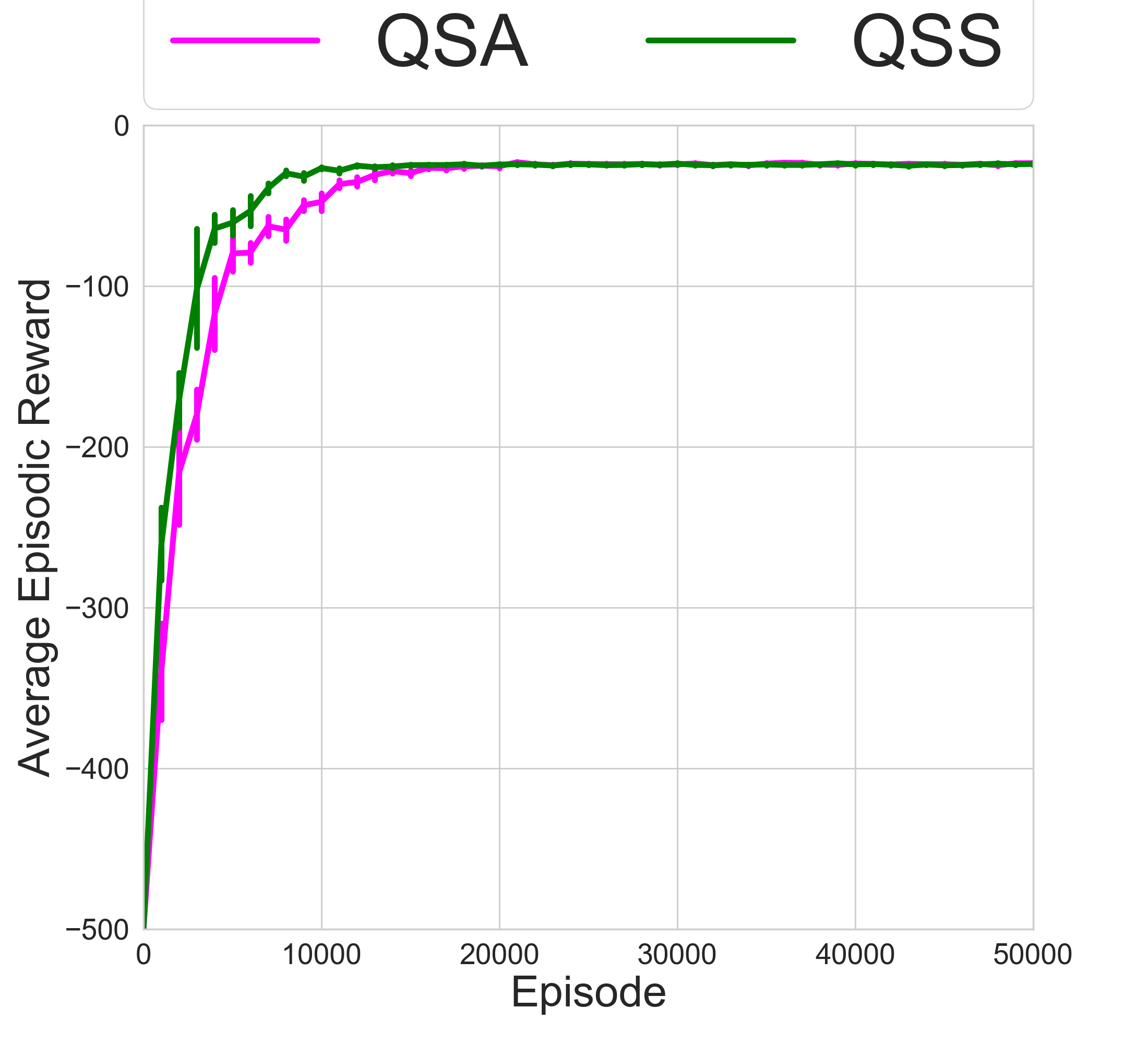}
    \caption{25\% }
  \end{subfigure}
    \begin{subfigure}{.325\linewidth}
  	\centering
    \includegraphics[width=\linewidth]{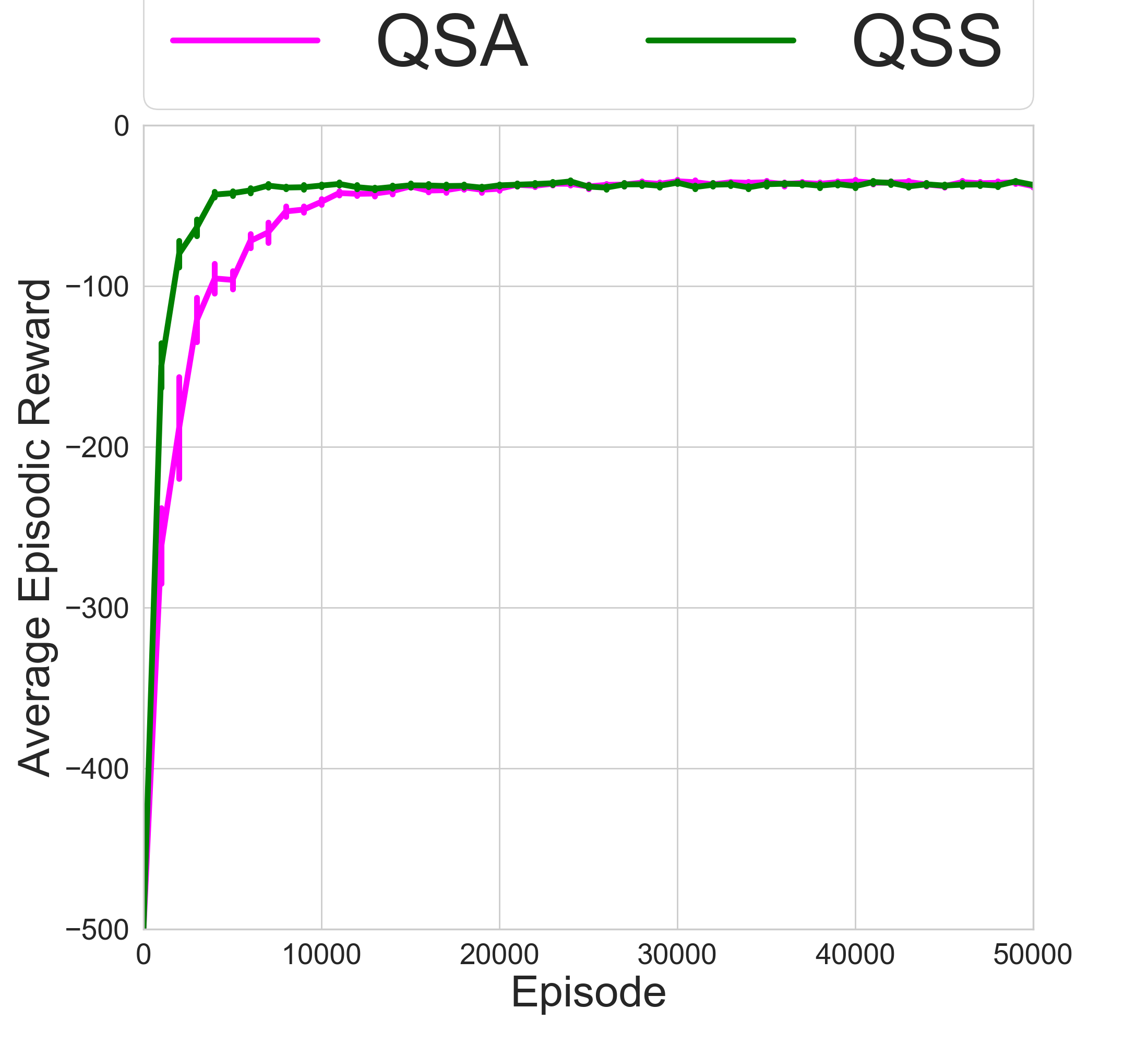}
    \caption{50\% }
  \end{subfigure}
  \centering
      \begin{subfigure}{.325\linewidth}
  	\centering
    \includegraphics[width=\linewidth]{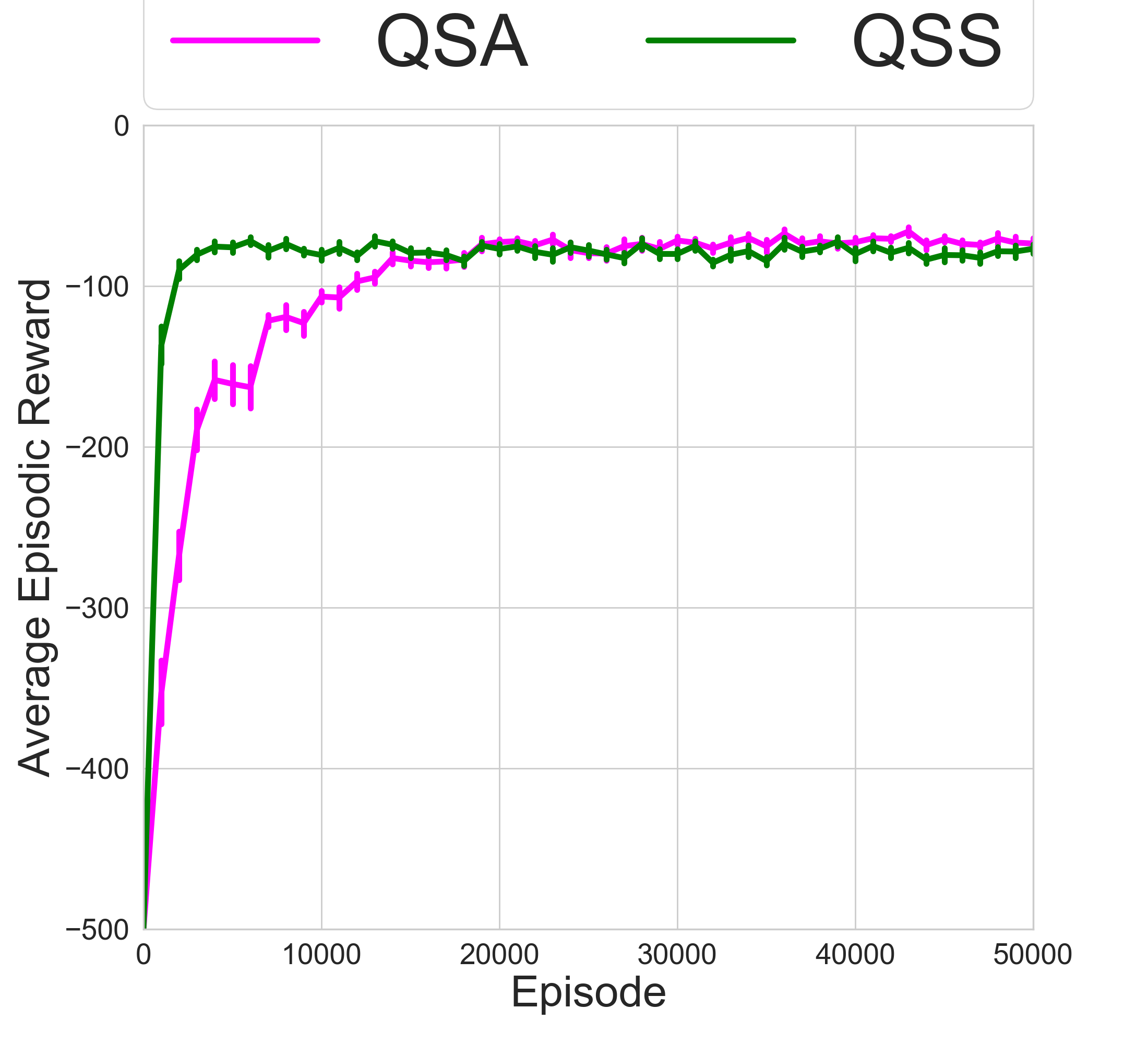}
    \caption{75\%}
  \end{subfigure}
    \caption{Stochastic experiments in an 11x11 gridworld. The first three experiments demonstrate the effect of stochastic actions on the average return. Before each episode, we evaluated the learned policy and averaged the return over 10 trials. All experiments were averaged over 10 seeds with 95\% confidence intervals.}
     \label{fig:stochastic_actions}
\end{figure}
\begin{figure}[!htb]
    \centering
  \begin{subfigure}{.49\linewidth}
  	\centering
    \includegraphics[width=.9\linewidth]{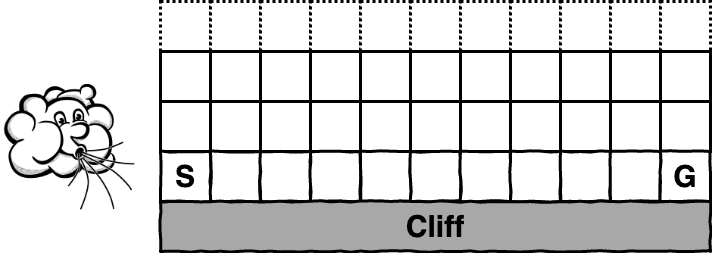}
  \end{subfigure}
  \begin{subfigure}{.49\linewidth}
  	\centering
    \includegraphics[width=.8\linewidth]{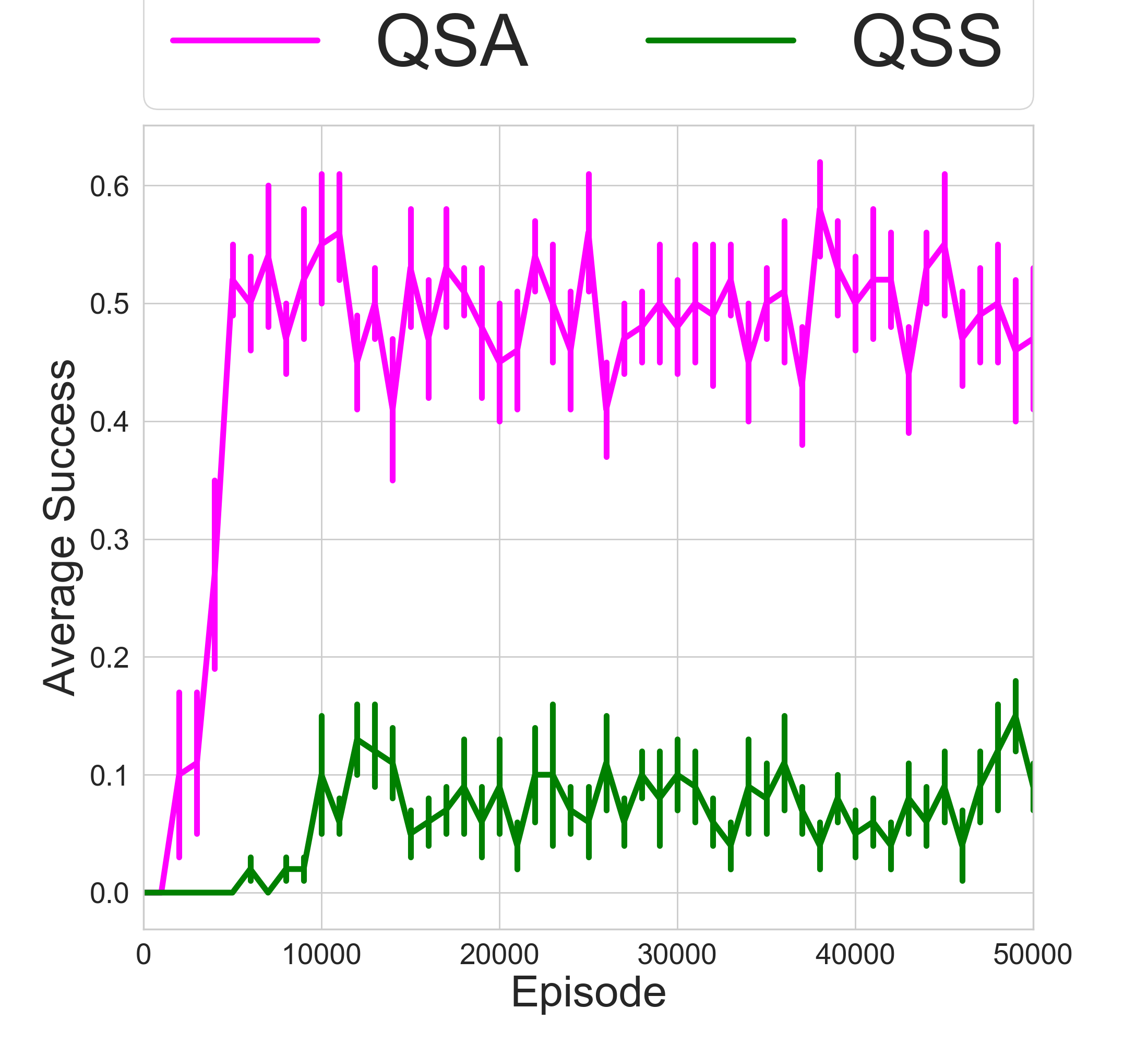}
  \end{subfigure}
    \caption{Stochastic experiments in cliffworld. This experiment measures the effect of stochastic actions on the average success rate.  Before each episode, we evaluated the learned policy and averaged the return over 10 trials. All experiments were averaged over 10 seeds with 95\% confidence intervals.}
     \label{fig:cliffworld}
\end{figure}
We were interested in measuring the impact of stochastic transitions on learning using QSS. To investigate this property, we add a probability of slipping to each transition, where the agent takes a random action (i.e. slips into an unintended next state) some percentage of time. Curiously, QSS solves this task quicker than QSA, even though it learns incorrect values (Figure~\ref{fig:stochastic_actions}). One hypothesis is that the slippage causes the agent to stumble into the goal state, which is beneficial for QSS because it directly updates values based on state transitions. The correct action that enables this transition is known using the given inverse dynamics model. QSA, on the other hand, would need to learn how the stochasticity of the environment affects the selected action's outcome and so the values may propagate more slowly.

We additionally study the case when stochasticity may lead to negative effects for QSS. We modify the gridworld to include a cliff along the bottom edge similar to the example in \citet{sutton1998reinforcement}. The agent is initialized on top of the cliff, and if it attempts to step down, it falls off and the episode is reset. Furthermore, the cliff is ``windy", and the agent has a $0.5$ probability of falling off the edge while walking next to it. The reward here is $0$ everywhere except the goal, which has a reward of $1$. Here, we see the effect of stochasticity is detrimental to QSS (Figure~\ref{fig:cliffworld}), as it does not account for falling and instead expects to transition towards the goal.

\subsection{Additional transfer experiment}
\begin{figure}[htb]
    \centering
    \includegraphics[width=.5\linewidth]{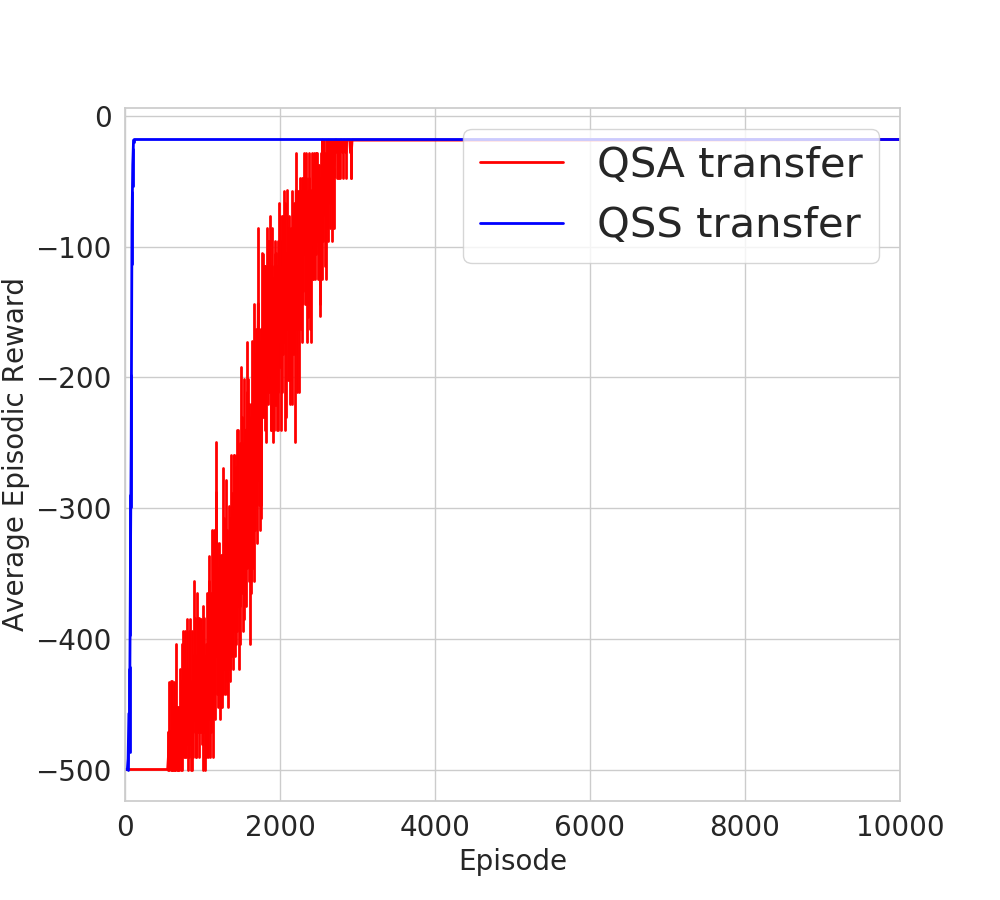}
    \caption{Transfer experiments within 11x11 gridworld.  The experiment represents how well QSS and QSA transfer to a gridworld with permuted actions. We now include an additional action that transports the agent back to the start. All experiments shown were averaged over 50 random seeds with 95\% confidence intervals
    }
    \label{fig:distinct_transfer}
\end{figure}
We trained QSA and QSS in a gridworld with an additional~\emph{transport} action that moved the agent back to the start. We then transferred the learned values to an environment where the action labels were shuffled. Incorrectly taking the transport action would have a larger impact on the average return than the other actions. QSS is able to learn much more quickly than QSA, as it only needs to relearn the inverse dynamics and avoids the negative impacts of the incorrectly labeled transport action.

\section{D3G Experiments}
\begin{table}[!htb]
\centering
\begin{tabular}{lllll}
\hline
\textbf{$\Theta$}        & \textbf{D3G} & \textbf{TD3} & \textbf{DDPG} & \textbf{BCO}\\
\hline
Critic lr           &  3e-4                         &  3e-4            &   3e-4  &  --         \\
Actor lr           &  --                        &  3e-4            &   3e-4   & --         \\
BC lr           &  --                        &  --            &   --   & 3e-4        \\
$\tau(s)$ lr        &  3e-4                         &  --            &   --  & --         \\
$f(s,\cdot)$ lr &      3e-4                 &      --        &       --    &  --   \\
$I(s,s')$ lr &          3e-4                 &        --       &        -- &  3e-4     \\
$\beta$                 &     1.0                      &   --           &   --    &   --      \\
$\eta$             &  0.005                        & 0.005             & 0.005      &   --    \\
Optimizer                      &  Adam                        &  Adam            & Adam   &  Adam          \\
Batch Size                     & 256                       &  256            &   256      &  256    \\
$\gamma$               & 0.99                         &  0.99            &  0.99  &  --         \\
Delay (d)            &   2                    &    2          &     --    &   --   \\ 
\hline
\end{tabular}
\caption{Hyperparameters $\Theta$ for D3G experiments.}
\label{tab:hyperparams}
\end{table}
We used the TD3 implementation from~\url{https://github.com/sfujim/TD3} for our experiments. We also used the ``OurDDPG" implementation of DDPG. We built our own implementation of D3G from this codebase. We used the default hyperparameters for all of our experiments, as described in Table~\ref{tab:hyperparams}. The replay buffer was filled for $10000$ steps before learning. All continuous experiments added noise $\epsilon \sim \mathcal{N}(0, 0.1)$ for exploration. In gridworld, the agent followed an $\epsilon$-greedy policy. Epsilon was set to $1$ and decayed to $.1$ by subtracting 9e-6 every time step. 

\subsection{Gridworld task}
We ran these experiments in an 11x11 gridworld. The state was the agent's $\langle x,y \rangle$ location on the grid. The agent was initialized to $\langle 0, 0 \rangle$ and received a reward of $-1$ until it reached the goal at $\langle 10, 10 \rangle$ and obtained a reward of $0$ and was reset to the initial position. The episode automatically reset after $500$ steps. 

\subsection{MuJoCo tasks}
We ran these experiments in the OpenAI Gym MuJoCo environment~\url{https://github.com/openai/gym}. We used gym==0.14.0 and mujoco-py==2.0.2. The agent's state was a vector from the MuJoCo simulator. 

\subsection{Learning from Observation Experiments}
We used TD3 to train an expert and used the learned policy to obtain demonstrations $D$ for learning from observation. We collected $1e6$ samples using the learned policy and took a random action either 0, 25, 50, 75, or 100 percent of the time, depending on the experiment. The samples consisted of the state, reward, next state, and done condition. 

We trained BCO with $D$ for $100$ iterations. During each iteration, we collected $1000$ samples from the environment using a Behavioral Cloning (BC) policy with added noise $\epsilon \sim \mathcal{N}(0, 0.1)$, then trained an inverse dynamics model for $10000$ steps, labeled the observational data using this model, then finally trained the BC policy with this labeled data for $10000$ steps.

We trained D3G with $D$ for $1e6$ time steps without any environment interactions. This allowed us to learn the model $\tau(s)$ which informed the agent of what state it should reach. Similarly to BCO, we used some environment interactions to train an inverse dynamics model for D3G. We ran this training loop for $100$ iterations as well. During each iteration, we collected $1000$ samples from the environment using the inverse dynamics policy $I(s,m(s))$ with added noise $\epsilon \sim \mathcal{N}(0, 0.1)$, then trained this model for $10000$ steps.

\section{Architectures}
\textbf{D3G Model $\tau(s)$}:

$s \rightarrow fc_{256} \rightarrow relu \rightarrow fc_{256}  \rightarrow relu \rightarrow fc_{len(s)}$

\textbf{D3G Forward Dynamics Model:}

$\langle s,a \rangle \rightarrow fc_{256} \rightarrow relu \rightarrow fc_{256}  \rightarrow relu \rightarrow fc_{len(s)}$

\textbf{D3G Forward Dynamics Model (Imitation):}

$\langle s,q \rangle \rightarrow fc_{256} \rightarrow relu \rightarrow fc_{256}  \rightarrow relu \rightarrow fc_{len(s)}$

\textbf{D3G Inverse Dynamics Model (Continuous):}

$\langle s,s' \rangle \rightarrow fc_{256} \rightarrow relu \rightarrow fc_{256}  \rightarrow relu \rightarrow fc_{len(a)} \rightarrow tanh \cdot$ max action

\textbf{D3G Inverse Dynamics Model (Discrete):}

$\langle s,s' \rangle \rightarrow fc_{256} \rightarrow relu \rightarrow fc_{256}  \rightarrow relu \rightarrow fc_{len(a)} \rightarrow softmax$

\textbf{D3G Critic:}
$\langle s,s' \rangle \rightarrow fc_{256} \rightarrow relu \rightarrow fc_{256}  \rightarrow relu \rightarrow fc_{l}$

\textbf{TD3 Actor:}

$s \rightarrow fc_{256} \rightarrow relu \rightarrow fc_{256}  \rightarrow relu \rightarrow fc_{len(a)} \rightarrow tanh \cdot$ max action

\textbf{TD3 Critic:}

$\langle s,a \rangle \rightarrow fc_{256} \rightarrow relu \rightarrow fc_{256}  \rightarrow relu \rightarrow fc_{l}$

\textbf{DDPG Actor:}

$s \rightarrow fc_{400} \rightarrow relu \rightarrow fc_{300}  \rightarrow relu \rightarrow fc_{len(a)} \rightarrow tanh \cdot$ max action

\textbf{DDPG Critic:}

$\langle s,a \rangle \rightarrow fc_{400} \rightarrow relu \rightarrow fc_{300}  \rightarrow relu \rightarrow fc_{l}$

\textbf{BCO Behavioral Cloning Model:}

$s \rightarrow fc_{256} \rightarrow relu \rightarrow fc_{256}  \rightarrow relu \rightarrow fc_{len(a)} \rightarrow tanh \cdot$ max action

\textbf{BCO Inverse Dynamics Model:}

$\langle s,s' \rangle \rightarrow fc_{256} \rightarrow relu \rightarrow fc_{256}  \rightarrow relu \rightarrow fc_{len(a)} \rightarrow tanh \cdot$ max action
\end{appendices}

\end{document}